\documentclass{article}
\pdfoutput=1
\usepackage[margin=1.2in]{geometry}
\usepackage[numbers]{natbib}
\usepackage{stfloats} 
\usepackage{framed}
\usepackage[most]{tcolorbox}
\usepackage{authblk}

\author[1]{Spyros Dragazis} \author[1,2]{Aldo Pacchiano} \affil[1]{Boston University} \affil[2]{Broad Institute of MIT and Harvard}

\usepackage{graphicx} 
\usepackage{amsmath,amscd,amssymb,amsthm}
\usepackage{thmtools}
\usepackage{thm-restate}
\declaretheorem[name=Theorem]{thm}
\newtheorem{assum}{Assumption}[section]
\newtheorem{lemma}{Lemma}
\newtheorem*{lemma*}{Lemma}  
\newtheorem{proposition}{Proposition}

\usepackage{algorithm}      
\usepackage{algpseudocode}  

\usepackage{enumitem}
\usepackage[dvipsnames]{xcolor}  
\usepackage{tcolorbox} 
\usepackage{manfnt} 
\usepackage{bbding}    

\usepackage[backref=page]{hyperref}
\hypersetup{
    colorlinks,
    citecolor=blue,
    filecolor=blue,
    linkcolor=blue,
    urlcolor=blue
}
\renewcommand*{\backref}[1]{}
\renewcommand*{\backrefalt}[4]{%
    \ifcase #1 %
    \or     #2%
    \else   #2%
    \fi
}

\usepackage{mymacros}

\usepackage{fontawesome5}

\usepackage{mathtools}

\usepackage{manfnt} 

\usepackage{parskip}

\usepackage{pifont} 

\usepackage{ulem} 

\usepackage{nicefrac}

\usepackage{cleveref}

\crefname{proposition}{Proposition}{Propositions}
\Crefname{proposition}{Proposition}{Propositions}

\crefname{assum}{Assumption}{Assumptions}
\Crefname{assum}{Assumption}{Assumptions}

\crefname{thm}{Theorem}{Theorems}
\Crefname{thm}{Theorem}{Theorems}

\crefname{defn}{Definition}{Definitions}
\Crefname{defn}{Definition}{Definitions}

\usepackage{subcaption}

\newcounter{constant} 
\newcommand{\newconstant}[1]{\refstepcounter{constant}\label{#1}} 
\newcommand{\useconstant}[1]{\kappa_{\ref{#1}}}

\usepackage{color-edits}

\usepackage{thm-restate} 

\addauthor[Spyros]{sd}{teal}
\addauthor[Aldo]{ap}{forestgreen}
\usepackage{mymacros}

\title{Safety by Design: Realized-Cost Constraints for Contextual Bandits with Continuous Actions}

\begin{document}

\maketitle

\begin{abstract}
Contextual bandits are a standard framework for sequential decision-making under uncertainty, with applications in clinical trials, dosage selection, recommendation systems, and autonomous systems.
Safety is central in many of these applications, since a single unsafe decision in settings such as dosage selection or autonomous driving can have catastrophic consequences.
A common way to model safety in bandit problems is to associate each action with both a reward signal and a cost signal, and to optimize reward subject to constraints on cost.
Most existing safety-constrained bandit models enforce safety by requiring the expected cost of each action to remain below a prescribed threshold.
However, this may be insufficient in heteroscedastic settings, where the chosen action affects not only the expected reward and cost, but also the variability of the observed outcomes.
We study contextual bandits with one-dimensional continuous actions and stage-wise high-probability constraints on the realized cost.
We propose High-Probability Constrained UCB, an optimistic-pessimistic algorithm that explores for reward while conservatively estimating the safe action set.
For linear reward and cost models, we prove a tight $\tilde{\CO}(d\sqrt{T})$ regret bound, and we extend the analysis to general function classes using the eluder dimension.
Experiments show that enforcing realized-cost safety substantially reduces violations compared with expected-cost constrained baselines.
\end{abstract}

\section{Introduction}
\label{sec:intro}

Bandit algorithms \citep{thompson1933likelihood,auer2002using,lattimore2020bandit} have been applied across a wide range of domains, including reinforcement learning with human feedback (RLHF) \citep{christiano2017deep,xie2024exploratory}, clinical trials \citep{matsuura2022optimal,lee2020contextual,aziz2021multi,villar2015multi,bastani2020online}, recommendation systems \citep{wang2022context}, dynamic pricing \citep{mueller2019low}, large language models \citep{bouneffouf2024tutorial}, and fair allocation \citep{procaccia2024honor}; see \cite{bouneffouf2019survey} for a broader overview.
In these domains, a learner interacts sequentially with an unknown environment, aiming to learn from observations while maximizing cumulative reward through its decisions.
Recently, there has been increasing interest in contextual bandits \citep{foster2023foundations, lattimore2020bandit}, where the learner observes a context before selecting an action from a potentially large or continuous action space.

Within this general framework, constrained bandit algorithms are particularly relevant when applications involve resource limitations or safety requirements \citep{dai2023safe,wachi2020safe}.
Such constraints can take different forms, including budget-type constraints as in knapsack bandits \citep{badanidiyuru2018bandits}, fairness constraints \citep{joseph2016fairness}, and models with simultaneous reward and cost signals \citep{amani2019linear,pacchiano2021stochastic,pacchiano2024contextual}.
The latter formulation is especially natural in applications such as advertising and drug administration, where decisions must balance utility against risk.
In drug dosage problems, for example, treatments generate an efficacy signal, interpreted as reward, and a toxicity signal, interpreted as cost; the goal is to maximize therapeutic benefit while ensuring that toxicity remains below a prescribed threshold $\tau$, as in Phase II clinical trials.
For example, related work by \cite{lee2020contextual} considers a binary reward--cost model that enforces constraints on the average cost over time.

While cumulative or averaged cost control is sufficient in some applications, many safety-critical settings require stage-wise constraints, namely control of the cost at each time step.
To address this requirement, researchers have studied the well-established safe linear bandit problem \citep{hutchinson2024directional,pacchiano2024contextual,moradipari2019safe}, in which, at each round $t$, the selected action generates both a reward signal $r_t \in \R$ and a cost signal $c_t \in \R$.
The objective is to maximize expected cumulative reward while ensuring that the expected cost remains below a known safety threshold $\tau$ at every round, with this constraint holding with high probability.

Guarantees on the expected value of the cost signal are natural in homoscedastic models, where the level of noise is not affected by the chosen action.
However, in applications such as autonomous driving, robotic control, and drug treatment, the variability of the outcome may itself depend on the action level.
For example, a faster robotic motion may lead to a wider distribution of possible next positions due to friction or sensor noise, a larger velocity in autonomous driving may be less predictable than a smaller one, and the effect of a drug may vary across dosage levels \cite{kpeglo2021identification}.
In such heteroscedastic settings, satisfying a constraint only in expectation does not preclude catastrophic outcomes at individual rounds.
This motivates algorithms that control the realized cost signal itself at each time step, rather than only its expectation.
With this motivation, we consider the following question.

\begin{center}
\textit{Is it possible to design constrained low-regret algorithms\\ that control realized cost at every round with high probability?}
\end{center}

In this work, we address this question in a heteroscedastic reward--cost model with one-dimensional continuous actions \cite{amani2019linear,pacchiano2024contextual,majzoubi2020efficient}.
Our goal is to control the realized cost signal $C_t$ at each time step by requiring it to satisfy the constraint $\P(C_t \le \tau) \ge 1-\delta$, where $\delta > 0$ is a user-specified confidence parameter.
As in prior work, the resulting guarantee accounts for both the noise in the cost observations and the internal randomization of the algorithm.

To achieve this objective, we propose a \textit{UCB-like} algorithm termed \textit{High Probability Constrained UCB}.
Our method is based on the \textit{Optimism in the Face of Uncertainty} (OFUL) principle \citep{auer2002using,abbasi2011improved} and does not require prior knowledge of an initial safe action, in contrast to several existing approaches \citep{pacchiano2024contextual}.
The algorithm applies in both stochastic and adversarial contextual settings, under the standard assumption of sub-Gaussian noise for both reward and cost signals.
Under this assumption, we construct a high-probability constraint event that ensures the realized cost remains below the safety threshold at every round.

We show that the proposed algorithm achieves a $T$-round regret bound of order $\tilde{\mathcal{O}}(d\sqrt{T})$, which is minimax optimal in its dependence on both the dimension $d$ and the time horizon $T$ for linear bandits \citep{lattimore2020bandit}.
In addition, we empirically validate the performance of our approach on both synthetic and real-world datasets.
Finally, we extend our analysis to settings with non-linear reward and cost functions, characterizing the regret in terms of the eluder dimension \citep{russo2013eluder}, a complexity measure for general function classes.

The main contributions of this work are as follows:
\begin{itemize}
    \item We introduce a safety framework for reward--cost constrained bandits in a heteroscedastic setting with one-dimensional continuous actions.
    \item We derive algorithms with tight regret guarantees for both linear and non-linear reward and cost models.
    \item We validate our method on both synthetic and real-world datasets, showing the importance of high-probability realized-cost guarantees in heteroscedastic settings.
\end{itemize}

\section{Model and Problem Formulation}
\label{sec:problemFormulation}

\textbf{Notation.}
We use the following notation throughout the paper.
For vectors $x,y \in \mathbb{R}^d$, we write $\langle x,y \rangle = x^\top y$ and $\langle x,y \rangle_{\mathbf{A}} = x^\top \mathbf{A} y$ for the standard and weighted inner products, respectively, where $\mathbf{A} \in \mathbb{R}^{d \times d}$ is positive definite.
Similarly, we define $\norm{x}_2 = \sqrt{\langle x,x \rangle}$ and $\norm{x}_{\mathbf{A}} = \sqrt{\langle x,x \rangle_{\mathbf{A}}}$ as the $\ell_2$ and weighted $\ell_2$ norms of $x$.
We denote the indicator function by $\indicator{\cdot}$.
Upper-case letters denote random variables, e.g., $X$, and lower-case letters denote their realizations, e.g., $X=x$.
We write $[T] = \{1,\ldots,T\}$.
Finally, $\widetilde{\mathcal O}$ denotes big-$\mathcal O$ notation up to logarithmic factors.

Inspired by bandit algorithms for adaptive dosage allocation \citep{matsuura2022optimal,lee2020contextual,aziz2021multi}, we consider the following formulation of actions, rewards, and costs.
At each round $t \in [T]$, the learner observes a $d$-dimensional context vector $X_t \in \mathbb{R}^d$, which may represent structured measurements, such as medical test results, or an embedding of unstructured information, such as text descriptions generated by a language model \citep{baheri2023llms}.
We impose no distributional assumptions on $X_t$; the contexts may be generated either stochastically or adversarially.
The learner then selects a one-dimensional continuous action $\alpha_t \in [0,1]$ \citep{majzoubi2020efficient}, which can be viewed as a continuous generalization of a finite set of dosage levels.
The environment generates reward and cost signals according to
\begin{align}
\label{eq:rewardCostSignal}
    R_t &:= g(\alpha_t) \cdot \left( r(X_t) + \xi_t^r \right), \\
    C_t &:= g(\alpha_t) \cdot \left( c(X_t) + \xi_t^c \right),
\end{align}
where $\xi_t^r$ and $\xi_t^c$ are sub-Gaussian noise terms, and $r(X_t)$ and $c(X_t)$ denote the underlying reward and cost functions.
The function $g:[0,1]\to[0,1]$ is known and strictly increasing, with $g(0)=0$ and $g(1)=1$.
In dosage applications, $g$ can be interpreted as a covariate-independent dose-response curve, such as a population-level efficacy or toxicity profile that can be estimated from prior studies.
\footnote{
A natural extension is to allow distinct dosage-response functions for the reward and cost, denoted by \(g_r(\cdot)\) and \(g_c(\cdot)\), respectively. While this extension is practically relevant, it introduces no additional theoretical difficulty, so we use a common function \(g(\cdot)\) to simplify the notation.
}
The normalization $g(1)=1$ is without loss of generality, since any constant scale can be absorbed into the reward and cost functions.
This formulation allows for nearly flat regions that model saturation or diminishing returns, while preserving invertibility.
In dosage-like applications, efficacy and toxicity are commonly assumed to increase with the dosage level; see \cite{aziz2021multi, lee2020contextual, le2009dose} and references therein.
In the linear setting, we assume $r(X_t) = \langle X_t, \theta^* \rangle$ and $c(X_t) = \langle X_t, \mu^* \rangle$ for unknown parameters $\theta^*, \mu^* \in \mathbb{R}^d$.
In more general settings, we require only that $r(X_t)$ and $c(X_t)$ are bounded.
The interaction protocol is summarized in \Cref{fig:Problem-Description}.

To illustrate the linear case, consider a medical treatment scenario in which the context $X_t \in \mathbb{R}^d$ represents a patient feature vector encoding medical attributes such as blood pressure, oxygen saturation, or laboratory test results.
In this model, there exist unknown parameters $\theta\opt, \mu\opt \in \mathbb{R}^d$ such that
\begin{align*}
    R_t &= g(\alpha_t) \cdot \left( \langle X_t, \theta\opt \rangle + \xi_t^r \right), \\
    C_t &= g(\alpha_t) \cdot \left( \langle X_t, \mu\opt \rangle + \xi_t^c \right).
\end{align*}
The vector $\theta\opt$ characterizes how patient-specific features influence treatment efficacy: a positive inner product $\langle X_t, \theta\opt \rangle$ indicates high expected benefit from treatment.
Similarly, the vector $\mu\opt$ captures how the same features relate to adverse effects or toxicity, with larger values of $\langle X_t, \mu\opt \rangle$ indicating increased risk of side effects.
The noise terms $\xi_t^r$ and $\xi_t^c$ represent stochastic variability not captured by the linear model, such as unobserved physiological factors or measurement errors.

The action $\alpha_t \in [0,1]$ represents the administered dosage level.
Through the known response function $g$, the dosage scales both the reward and cost signals, allowing the learner to trade off efficacy against toxicity.
For instance, a patient with elevated inflammatory markers, corresponding to a high value of $\langle X_t, \theta\opt \rangle$, may benefit substantially from treatment.
However, if the same patient exhibits impaired kidney function, corresponding to a high value of $\langle X_t, \mu\opt \rangle$, a reduced dosage may be preferable to mitigate adverse outcomes.
This formulation is consistent with standard monotonicity assumptions on efficacy--toxicity dose-response curves \cite{wages2011dose, le2009dose}, whereby increasing the dose tends to increase both therapeutic and toxic effects.

The same action-dependent scaling also makes the model heteroscedastic.
Indeed, the variability of the observed reward and cost depends on the selected dosage through $g(\alpha_t)$.
This captures the fact that aggressive treatments may lead to less predictable responses \citep{rawlins1974variability}.
In this sense, the model allows for stochastic departures from deterministic monotonicity, such as saturation or instability in biological responses.

As a result, the learner must select actions $\alpha_t$ not only to maximize expected cumulative reward, but also to control the variability induced by the treatment.
Our model belongs to the broader class of heteroscedastic bandits \citep{xu2023noise, mukherjee2024speed, weltz2023experimental, hsieh2019stay}, in which the noise distribution may depend on both the current context and the chosen action.
In more general formulations, the noise variance or sub-Gaussian parameter may scale as a quadratic form $\phi(X_t,A_t)^\top \Sigma\opt \phi(X_t,A_t)$, where $\phi(\cdot,\cdot): \CX \times \CA \to \mathbb{R}^d$ denotes a joint context--action embedding and $\Sigma\opt$ is an unknown positive definite matrix.
While our model adopts a simpler structure, it captures monotone dose-response behavior commonly observed in reward--cost systems and remains interpretable, a property that is particularly valuable in medical decision-making contexts.

\begin{figure}[t]
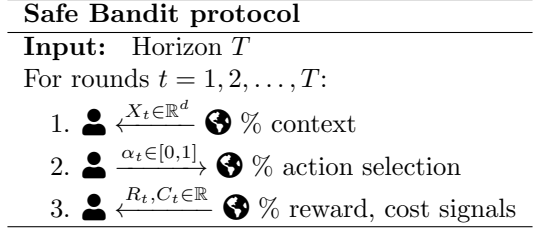

\centering
\begin{tabular}{l}
\hline
\textbf{Safe Bandit protocol} \\
\hline
\textbf{Input: } Horizon $T$ \\
For rounds $t = 1, 2, \ldots, T$: \\
\quad 1. {\learner} $\xleftarrow{X_t \in \mathbb{R}^d}$ {\env} \% context \\
\quad 2. {\learner} $\xrightarrow{\alpha_t \in [0,1]}$ {\env} \% action selection \\
\quad 3. {\learner} $\xleftarrow{R_t, C_t \in \mathbb{R}}$ {\env} \% reward, cost signals \\
\hline
\end{tabular}
\caption{Safe Bandit protocol.}
\label{fig:Problem-Description}
\end{figure}

\subsection{Assumptions}
\label{sec:assumptions}

Before presenting our analysis and results, we state the assumptions used throughout the paper.
These assumptions are standard in the literature on constrained contextual bandits; see \cite{hutchinson2024directional,pacchiano2024contextual,moradipari2019safe} and references therein.

\newconstant{Cr}
\newconstant{Cc}

\begin{assum}[Sub-Gaussian noise]
\label{ass:noise-sub-gaussian}
For all $t\in[T]$, the reward and cost noise random variables $\xi_t^r$ and $\xi_t^c$ are conditionally sub-Gaussian.
That is, for all $\lambda\in\mathbb R$, there exist constants $\useconstant{Cr}, \useconstant{Cc} >0$ such that
\begin{align*}
&\mathbb{E}[\xi_t^r \mid \mathcal{H}_{t-1}] = 0,
\qquad
\mathbb{E}[\exp(\lambda \xi_t^r) \mid \mathcal{H}_{t-1}]
\leq \exp(\lambda^2 \useconstant{Cr}^2/2), \\
&\mathbb{E}[\xi_t^c \mid \mathcal{H}_{t-1}] = 0,
\qquad
\mathbb{E}[\exp(\lambda \xi_t^c) \mid \mathcal{H}_{t-1}]
\leq \exp(\lambda^2 \useconstant{Cc}^2/2),
\end{align*}
where $\mathcal H_t :=
\sigma\!\left(
X_1,\alpha_1,R_1,C_1,\ldots,
X_t,\alpha_t,R_t,C_t
\right).
$ is the filtration containing the history up to the end of round $t$.
\end{assum}

\newconstant{S}
\begin{assum}[Bounded parameters]
\label{ass:bounded-reward-cost-param}
There exists a known constant $\useconstant{S} > 0$ such that $\|\theta\opt\|_2 \leq \useconstant{S}$ and $\|\mu\opt\|_2 \leq \useconstant{S}$.
\end{assum}

\newconstant{L}
\begin{assum}[Bounded contexts]
\label{ass:bounded-contexts}
The $\ell_2$-norms of all contexts are bounded by a known constant $\useconstant{L} > 0$, i.e.,
\begin{equation*}
\max_{t\in[T]}\;\|X_t\|_2 \leq \useconstant{L}.   
\end{equation*}
\end{assum}

\begin{assum}[Positive toxicity threshold]
\label{ass:positive-threshold}
The toxicity threshold satisfies $\tau >0$.
\end{assum}

\subsection{Constraint Formulation}
\label{sec:constraint}

The learner's goal is to maximize expected cumulative reward over $T$ rounds,
\[
\mathbb{E}\!\left[\sum_{t=1}^T g(\alpha_t)\langle X_t,\theta\opt\rangle\right],
\]
while ensuring that the realized cost remains below a known safety threshold with high probability.
Our algorithm takes as input a confidence level $\delta$, which controls the probability of violating the constraint due to stochastic fluctuations in the cost signal.
Thus, unlike formulations that constrain only the expected cost, our safety requirement depends explicitly on the tail behavior of the realized cost.

If the noise distribution were fully known, for instance Gaussian with known variance, one could potentially enforce the constraint directly using the corresponding quantiles.
In our setting, however, the noise distribution is unknown and we assume only a sub-Gaussian upper bound.
We therefore rely on distribution-free high-probability bounds and formulate the following sufficient condition.

\begin{restatable}{lemma}{constraintFormulation}
\label{lem:constraintFormulation}
If the selected action $\alpha_t$ satisfies
\[
g(\alpha_t)\left(
\langle X_t,\mu\opt\rangle
+
\useconstant{Cc}\sqrt{2\log\left(\tfrac{1}{\delta}\right)}
\right)
\leq \tau,
\]
then $\P(C_t \leq \tau \mid \mathcal{H}_{t}) \geq 1-\delta$.
\end{restatable}

The proof is provided in \Cref{app:constraint} and follows directly from a concentration inequality for sub-Gaussian random variables; see \cite{lattimore2020bandit}, Chapter~5, or \cite{vershynin2018high}, Chapter~2.

The key difference from standard expected-cost formulations is conceptual rather than merely algebraic: the constraint is imposed on the realized cost, and the confidence parameter $\delta$ specifies the tolerated probability of an unsafe realization. In our action-scaled heteroscedastic model, the selected action controls both the mean cost and the magnitude of the cost fluctuations, making this realized-cost requirement a meaningful decision constraint.

Accordingly, we define the feasible action set at round $t$ as
\begin{equation}
\label{eq:setOfFeasibleDosages}
    \mathcal{A}^{f}_t(\delta)
    :=
    \left\{
    \alpha \in [0,1]:
    g(\alpha)
    \left(
    \langle X_t,\mu\opt\rangle
    +
    \useconstant{Cc}\sqrt{2\log\left(\tfrac{1}{\delta}\right)}
    \right)
    \leq \tau
    \right\}.
\end{equation}
Any action $\alpha_t \in \mathcal{A}^{f}_t(\delta)$ ensures that
$\P(C_t \leq \tau \mid \mathcal{H}_t) \geq 1-\delta$.

Given the feasible action set, let
\[
\alpha_t\opt
\in
\arg\max_{\alpha\in\mathcal{A}^{f}_t(\delta)}
g(\alpha)r(X_t)
\]
denote an optimal feasible action at round $t$.\footnote{
In prior work on contextual bandits with stage-wise safety constraints
\citep{pacchiano2024contextual,amani2019linear}, the learner's performance is typically
measured relative to the optimal safe action.
This choice of comparator reflects the requirement that the safety constraint be satisfied
at every round with high probability; accordingly, the resulting regret guarantees are also
stated with high probability.
}
We measure the learner's performance through the $T$-round constrained pseudo-regret
\begin{equation}\label{eq:reg_def}
\mathcal{R}_{\mathcal{C}}(T)
:=
\sum_{t=1}^{T}
\left[
g(\alpha_t\opt)r(X_t)
-
g(\alpha_t)r(X_t)
\right].
\end{equation}

Our algorithm guarantees that the selected action belongs to an estimated subset of $\mathcal{A}^{f}_t(\delta)$ with probability at least $1-\delta'$, where $\delta'$ is the confidence level associated with estimating $\mu\opt$.
We keep $\delta$ and $\delta'$ separate to distinguish two sources of uncertainty: the randomness of the realized cost, governed by the noise $\xi_t^c$, and the statistical uncertainty in the estimate of $\mu\opt$, governed by the observed history $\mathcal{H}_t$.
One may combine these events by a union bound, but we state them separately to make the two roles explicit.

We now address safety in the initial round, where no observations are available to estimate $\theta\opt$ or $\mu\opt$.
In contrast to prior work \citep{pacchiano2024contextual,moradipari2019safe}, our model does not require prior knowledge of an initially safe action.
Because $\tau>0$ and $g(0)=0$, the zero action is always safe.
Moreover, using the boundedness assumptions on $\mu\opt$ and $X_t$, we can characterize a nontrivial initial safe interval.
Specifically, since $\langle X_1,\mu\opt\rangle \leq \useconstant{S}\useconstant{L}$ by Cauchy--Schwarz, any action satisfying
\begin{equation}\label{eq:initialSafeAction}
g(\alpha_1)
\leq
\frac{\tau}{
\useconstant{Cc}\sqrt{2\log\left(\tfrac{1}{\delta}\right)}
+
\useconstant{S}\useconstant{L}
}.
\end{equation}
is safe.
Equivalently, since $g$ is strictly increasing, an initial safe interval is given by
\[
\mathcal{A}^{f}_1(\delta)
=
\left[
0,
g^{-1}
\left(
\min\left\{
g(1),
\frac{\tau}{
\useconstant{Cc}\sqrt{2\log\left(\tfrac{1}{\delta}\right)}
+
\useconstant{S}\useconstant{L}
}
\right\}
\right)
\right].
\]

\section{Proposed Algorithm}
\label{sec:algo_high_prob}

Our algorithm follows the principle of optimism in the face of uncertainty for the reward signal, while using pessimistic estimates for the cost signal.
This asymmetric treatment encourages exploration for reward maximization while maintaining conservative estimates of the safe action set.

At each round, we construct two regularized least-squares estimators, one for $\theta\opt$ and one for $\mu\opt$.
For a regularization parameter $\lambda>0$, define the regularized covariance matrix at round $t$ as
\begin{equation}
\label{eq:reg_covariance_m}
\Sigma_t
=
\lambda I
+
\sum_{\substack{s<t\\ g(\alpha_s)>0}}
X_sX_s^\top .
\end{equation}
Our theoretical analysis holds for any fixed constant $\lambda>0$; in our experiments, we set $\lambda=1$.
Using \Cref{eq:reg_covariance_m}, we define
\begin{equation}
\label{eq:LSE_theta_mu}
\hat{\theta}_t
=
\Sigma_t^{-1}
\sum_{\substack{s<t\\ g(\alpha_s)>0}}
\frac{R_s}{g(\alpha_s)}X_s,
\qquad
\hat{\mu}_t
=
\Sigma_t^{-1}
\sum_{\substack{s<t\\ g(\alpha_s)>0}}
\frac{C_s}{g(\alpha_s)}X_s .
\end{equation}
Since
\[
\frac{R_s}{g(\alpha_s)}
=
\langle X_s,\theta\opt\rangle+\xi_s^r,
\qquad
\frac{C_s}{g(\alpha_s)}
=
\langle X_s,\mu\opt\rangle+\xi_s^c,
\]
the transformed observations follow the standard linear regression model whenever $g(\alpha_s)>0$.
Rounds with $g(\alpha_s)=0$ provide no information about the unknown parameters and are therefore excluded.

To construct a UCB-style algorithm, we define high-probability confidence sets centered at $\hat{\theta}_t$ and $\hat{\mu}_t$.
These sets bound the estimation error with respect to the true parameters $\theta\opt$ and $\mu\opt$ and follow from standard self-normalized concentration inequalities.
Let
\[
N(t)=\left|\{1\le s<t:g(\alpha_s)>0\}\right|
\]
denote the number of informative samples collected before round $t$.

\begin{thm}[Theorem~2 in~\citealp{abbasi2011improved}]
\label{thm:yasin_theorem}
For fixed $\delta'\in(0,1)$, define
\begin{align*}
\beta_t^r(\delta',d)
&=
\useconstant{Cr}
\sqrt{
d\log\left(
\frac{1+N(t)\useconstant{L}^2/\lambda}{\delta'}
\right)
}
+
\sqrt{\lambda}\useconstant{S},
\\
\beta_t^c(\delta',d)
&=
\useconstant{Cc}
\sqrt{
d\log\left(
\frac{1+N(t)\useconstant{L}^2/\lambda}{\delta'}
\right)
}
+
\sqrt{\lambda}\useconstant{S}.
\end{align*}
Then, with probability at least $1-\delta'$,
\[
\|\hat{\theta}_t-\theta\opt\|_{\Sigma_t}
\le
\beta_t^r(\delta',d),
\qquad
\|\hat{\mu}_t-\mu\opt\|_{\Sigma_t}
\le
\beta_t^c(\delta',d).
\]
\end{thm}

We use \Cref{thm:yasin_theorem} to define the confidence ellipsoids
\begin{equation}
\label{eq:assymetric-confidence-sets}
\begin{split}
\mathcal{C}_t^r(\delta')
&=
\left\{
\theta\in\mathbb{R}^d:
\|\theta-\hat{\theta}_t\|_{\Sigma_t}
\le
\beta_t^r(\delta',d)
\right\},
\\
\mathcal{C}_t^c(\delta')
&=
\left\{
\mu\in\mathbb{R}^d:
\|\mu-\hat{\mu}_t\|_{\Sigma_t}
\le
\beta_t^c(\delta',d)
\right\}.
\end{split}
\end{equation}
When clear from context, we omit the dependence of $\beta_t^r$ and $\beta_t^c$ on $\delta'$ and $d$.
By a union bound, the events $\theta\opt\in\mathcal{C}_t^r$ and $\mu\opt\in\mathcal{C}_t^c$ hold with probability at least $1-2\delta'$.

We choose an optimistic estimate of the reward parameter by solving
\begin{equation}
\label{eq:theta_hat}
\tilde{\theta}_t
\in
\arg\max_{\theta\in\mathcal{C}_t^r}
\langle X_t,\theta\rangle .
\end{equation}
This is a linear optimization problem over an ellipsoid and admits a closed-form solution from the KKT conditions \citep{boyd2004convex}; see \Cref{sec:KKT_derivation} for details.

For the cost parameter, we take the opposite approach.
We select a pessimistic estimate $\tilde{\mu}_t$ that yields a conservative estimate of the feasible action set.
Since the feasible set is an interval in $[0,1]$, this amounts to choosing the parameter in the confidence set that makes the upper endpoint of the safe interval as small as possible.
This optimistic--pessimistic principle underlies \Cref{alg:HPUCB}.

We have stated the algorithm using the sub-Gaussian parameters
$\useconstant{Cr}$ and $\useconstant{Cc}$.
If these parameters are unknown, they can be replaced by valid upper confidence estimates or upper bounds using standard techniques.
In particular, the confidence intervals of \citep{jun2024noise} can be incorporated without changing the structure of the algorithm.
Moreover, when reward and cost observations are almost surely bounded, a sub-Gaussian upper bound follows directly from standard concentration results \citep{vershynin2018high}.

\begin{algorithm}[H]
\caption{High Probability Constrained UCB}
\label{alg:HPUCB}
\begin{algorithmic}[1]
\Require Constraint threshold $\tau>0$, safety parameter $\delta$, confidence parameter $\delta'$, sub-Gaussianity constants $\useconstant{Cr},\useconstant{Cc}$
\State $\alpha_1 \leftarrow
g^{-1}\!\left(
\min\left\{
1,
\frac{\tau}{
\useconstant{Cc}\sqrt{2\log(1/\delta)}
+
\useconstant{S}\useconstant{L}
}
\right\}
\right)$
\State Select action $\alpha_1$ and store $(X_1,R_1,C_1)$
\For{$t=2,3,\ldots,T$}
    \State Compute $\hat{\theta}_t,\hat{\mu}_t$ according to~\eqref{eq:LSE_theta_mu}.
    \State Use $\hat{\mu}_t$ to solve~\eqref{eq:mu_convex_prog} and compute $\hat{\mathcal{A}}^f_t$ using~\eqref{eq:feasible_set}.
    \State Use $\hat{\theta}_t$ to solve~\eqref{eq:theta_hat} and compute $\tilde{\theta}_t$.
    \State Compute
    \[
    \alpha_t
    \in
    \arg\max_{\alpha\in\hat{\mathcal{A}}^f_t}
    g(\alpha)\langle X_t,\tilde{\theta}_t\rangle .
    \]
    \State Select action $\alpha_t$.
    \If{$g(\alpha_t)>0$}
        \State Store $(X_t,R_t,C_t)$.
    \EndIf
\EndFor
\end{algorithmic}
\end{algorithm}

\subsection{Construction of the Estimated Feasible Set}
\label{sec:choice_mu}

The estimated feasible set $\hat{\mathcal{A}}^f_t$ is constructed from a pessimistic estimate of the cost parameter.
First, we compute the least-squares estimator $\hat{\mu}_t$, which defines a confidence ellipsoid that contains $\mu\opt$ with high probability.
Among all parameters in this ellipsoid, we choose the one that makes the safe action set most conservative.

More precisely, we construct a surrogate feasible set $\hat{\mathcal{A}}^f_t$ such that
\[
\P\!\left(
\hat{\mathcal{A}}^f_t
\subseteq
\mathcal{A}^f_t
\right)
\ge
1-\delta'
\]
for all $t$.
We compute $\hat{\mu}_t$ according to~\eqref{eq:LSE_theta_mu} and solve the following optimization problem:
\begin{equation}
\label{eq:mu_convex_prog}
\begin{aligned}
    & \underset{\mu}{\max}
    && \langle X_t,\mu\rangle \\
    & \text{subject to}
    && \|\mu-\hat{\mu}_t\|_{\Sigma_t}
    \leq
    \beta_t^c, \\
    &&
    & \langle X_t,\mu\rangle
    +
    \useconstant{Cc}\sqrt{2\log\left(\frac{1}{\delta}\right)}
    \geq 0 .
\end{aligned}
\end{equation}
The last constraint ensures that the denominator defining the safe upper endpoint is nonnegative.
Let
\begin{equation}
\label{eq:convexProgramSols}
\KappaM_t
=
\left\{
\mu\in\mathbb{R}^d:
\|\mu-\hat{\mu}_t\|_{\Sigma_t}
\leq
\beta_t^c,
\;
\langle X_t,\mu\rangle
+
\useconstant{Cc}\sqrt{2\log\left(\frac{1}{\delta}\right)}
\geq 0
\right\}.
\end{equation}
If $\KappaM_t\neq\emptyset$, let
\[
\tilde{\mu}_t
\in
\arg\max_{\mu\in\KappaM_t}
\langle X_t,\mu\rangle .
\]
We then define the estimated feasible set as
\begin{equation}
\label{eq:feasible_set}
\hat{\mathcal{A}}^f_t(\delta)
=
\begin{cases}
[0,1],
&
\text{if } \KappaM_t=\emptyset,
\\[0.5em]
\left[
0,
g^{-1}\!\left(
\min\left\{
1,
\dfrac{\tau}{
\langle X_t,\tilde{\mu}_t\rangle
+
\useconstant{Cc}\sqrt{2\log\left(\frac{1}{\delta}\right)}
}
\right\}
\right)
\right],
&
\text{if } \KappaM_t\neq\emptyset.
\end{cases}
\end{equation}

Combining \Cref{thm:yasin_theorem} with the pessimistic construction above yields the following guarantee; its proof is provided in the supplementary material.

\begin{restatable}{lemma}{feasibleSet}
\label{lem:feasibleSet}
For all $t\in[T]$,
\[
\P\left(
\hat{\mathcal{A}}^f_t(\delta)
\subseteq
\mathcal{A}^f_t(\delta)
\mid
\mathcal{H}_t
\right)
\ge
1-\delta' .
\]
\end{restatable}

\section{The ``Good Event''}
\label{sec:goodEvent}

A common analysis technique in multi-armed bandit algorithms is to condition on a high-probability ``good event.''
We define this event as the conjunction of the concentration and feasibility guarantees required by our algorithm.
Specifically, let $\Varepsilon$ denote the intersection of the following sub-events:
\begin{itemize}
    \item $\Varepsilon_{\theta} = \{\theta\opt \in \CC^r_t,\ \forall t \in [T]\}$.
    \item $\Varepsilon_{\mu} = \{\mu\opt \in \CC^c_t,\ \forall t \in [T]\}$.
    \item $\Varepsilon_{A} = \{\EstFeaSet \subseteq \FeaSet,\ \forall t \in [T]\}$.
\end{itemize}
The first two events ensure that the confidence sets for $\theta\opt$ and $\mu\opt$ contain the true parameters, as established in \Cref{thm:yasin_theorem}.
The third event ensures that the estimated feasible action set is contained in the true feasible action set at every round.

\begin{restatable}{lemma}{goodEvent}
\label{lem:goodEvent}
Let $\Varepsilon = \Varepsilon_{\theta} \cap \Varepsilon_{\mu} \cap \Varepsilon_{A}$.
Then,
\begin{equation*}
    \P(\Varepsilon) \ge 1 - 3\delta'.
\end{equation*}
\end{restatable}

\section{Regret Analysis -- Linear Case}
\label{sec:reg_analysisLinear}

We remind the definition of the $T$-round constrained pseudo-regret as defined in \Cref{eq:reg_def}.
In the linear setting, the expected reward obtained by action $\alpha_t$ is
$g(\alpha_t)\langle X_t,\theta\opt\rangle$.
We therefore define
\begin{equation}
\label{eq:regret_linear}
\mathcal{R}_{\mathcal{C}}(T)
=
\sum_{t=1}^{T}
\left(
r^{*}(X_t)-r_t
\right),
\end{equation}
where
\[
r^{*}(X_t)
=
\max_{\alpha\in\mathcal{A}^{f}_t}
g(\alpha)\langle X_t,\theta\opt\rangle
\]
denotes the optimal feasible reward at round $t$
, and
\[
r_t
=
g(\alpha_t)\langle X_t,\theta\opt\rangle
\]
is the expected reward obtained by the algorithm.

Since $g$ is strictly increasing, the optimal feasible action depends on the sign of
$\langle X_t,\theta\opt\rangle$.
If $\langle X_t,\theta\opt\rangle\ge 0$, the optimal action is the largest feasible action.
If $\langle X_t,\theta\opt\rangle<0$, the optimal action is $\alpha_t=0$.
Let $\alpha_t\opt$ denote an optimal feasible action at round $t$.

Then
\begin{equation}
\label{eq:regret_init}
\begin{aligned}
\mathcal{R}_{\mathcal{C}}(T)
&=
\sum_{t=1}^{T}
\max_{\alpha\in\mathcal{A}^{f}_t}
g(\alpha)\langle X_t,\theta\opt\rangle
-
g(\alpha_t)\langle X_t,\theta\opt\rangle
\\
&=
\sum_{t=1}^{T}
\left(
g(\alpha_t\opt)-g(\alpha_t)
\right)
\langle X_t,\theta\opt\rangle .
\end{aligned}
\end{equation}

We adopt a regret decomposition similar to those commonly used in the literature on constrained linear bandits
\citep{amani2019linear,pacchiano2021stochastic,pacchiano2024contextual}.
Define the auxiliary feasible action
\begin{equation}
\label{eq:alpha_tilde}
\tilde{\alpha}_t
\in
\arg\max_{\alpha\in\mathcal{A}^{f}_t}
g(\alpha)\langle X_t,\tilde{\theta}_t\rangle,
\end{equation}
where $\tilde{\theta}_t$ is the optimistic reward parameter defined in \Cref{eq:theta_hat}.
Using this definition, we decompose the regret as
\begin{equation}
\label{eq:regret_decomposition}
\begin{aligned}
\mathcal{R}_{\mathcal{C}}(T)
&=
\sum_{t=1}^{T}
\left(
g(\alpha_t\opt)-g(\alpha_t)
\right)
\langle X_t,\theta\opt\rangle
\\
&=
\sum_{t=1}^{T}
\underset{\text{Term 1: cost of estimating $\theta\opt$}}
{
\left(
g(\alpha_t\opt)-g(\tilde{\alpha}_t)
\right)
\langle X_t,\theta\opt\rangle
}
\\
&\quad+
\sum_{t=1}^{T}
\underset{\text{Term 2: cost of estimating $\mu\opt$}}
{
\left(
g(\tilde{\alpha}_t)-g(\alpha_t)
\right)
\langle X_t,\theta\opt\rangle
}.
\end{aligned}
\end{equation}

\subsection{Analysis of the Regret}
\label{section::regret_analysis}

\begin{restatable}{lemma}{regFirstTermBound}
\label{lem:reg_first_term_bound}
For all $t\in[T]$, on the event $\Varepsilon_\theta$, the first term in the regret decomposition satisfies
\[
\left(
g(\alpha_t\opt)-g(\tilde{\alpha}_t)
\right)
\langle X_t,\theta\opt\rangle
\le
g(\tilde{\alpha}_t)
\langle X_t,\tilde{\theta}_t-\theta\opt\rangle .
\]
\end{restatable}

The proof is provided in the supplementary material.
This bound is the analogue of the standard optimism argument for linear bandits
\citep{abbasi2011improved}, applied to the transformed action value $g(\alpha)$.
The main additional step is to control the second term, which captures the loss due to using a conservative estimate of the feasible action set.

\begin{restatable}{lemma}{regSecTermBound}
\label{lem:reg_sec_term_bound}
For all $t\in[T]$, on the event $\Varepsilon_\mu$, the second term in the regret decomposition satisfies
\[
\left(
g(\tilde{\alpha}_t)-g(\alpha_t)
\right)
\langle X_t,\theta\opt\rangle
\le
\useconstant{S}\useconstant{L}
\cdot
\frac{
\langle X_t,\tilde{\mu}_t-\mu\opt\rangle
}{\tau}.
\]
\end{restatable}

By combining \Cref{lem:reg_first_term_bound} and \Cref{lem:reg_sec_term_bound} with the regret decomposition in \Cref{eq:regret_decomposition}, we obtain the following bound.

\begin{restatable}{thm}{regretBoundLinear}
\label{theorem:regret_bound}
With probability at least $1-\delta'$, the regret of the \textit{High Probability Constrained UCB} algorithm satisfies
\[
\begin{aligned}
\mathcal{R}_{\mathcal{C}}(T)
=
\CO\!\Bigl(
&
\Bigl(1+\tfrac{\useconstant{S}\useconstant{L}}{\tau}\Bigr)
\cdot
\max\{\beta^r_T(\delta',d),\beta^c_T(\delta',d)\}
\\
&\times
\sqrt{
dT\log\!\left(
1+\tfrac{T\useconstant{L}^2}{d}
\right)
}
\Bigr).
\end{aligned}
\]
\end{restatable}

All proofs are deferred to the supplementary material.
Notably, the bound does not depend on $\delta$, since the benchmark is the clairvoyant policy that knows $\theta\opt$ and $\mu\opt$ but must satisfy the same realized-cost constraint as the learner.

\section{Non-linear Rewards and Costs}
\label{sec:non-linear}

We now extend the analysis beyond linear reward and cost models.
Instead of assuming that the reward and cost functions are linear in unknown parameters, we consider general function classes and express the regret bound in terms of the \textit{eluder dimension} \citep{russo2013eluder}.

At each round $t$, the learner observes a context $X_t$ and selects an action $\alpha_t\in[0,1]$.
The reward and cost signals are given by
\[
R_t
=
g(\alpha_t)\left(\theta\opt(X_t)+\xi_t^r\right),
\qquad
C_t
=
g(\alpha_t)\left(\mu\opt(X_t)+\xi_t^c\right),
\]
where $\theta\opt\in\CF^r$ and $\mu\opt\in\CF^c$ are the unknown mean reward and cost functions, respectively.
The function classes $\CF^r$ and $\CF^c$ are known to the learner.
We assume that all functions in these classes are bounded in $[-1,1]$.
This relaxes the common normalization to $[0,1]$; the important property for the analysis is boundedness.
As in the linear case, the noise variables $\xi_t^r$ and $\xi_t^c$ are conditionally sub-Gaussian.

The realized-cost feasible action set at round $t$ is
\begin{equation}
\label{eq:nonlinear_feasible_set_true}
\mathcal{A}_t^f(X_t)
=
\left\{
\alpha\in[0,1]:
g(\alpha)
\left(
\mu\opt(X_t)
+
\useconstant{Cc}\sqrt{2\log\left(\frac{1}{\delta}\right)}
\right)
\le
\tau
\right\}.
\end{equation}
The agent selects actions from an estimated subset of this feasible set.

For a subset $\widetilde{\CF}\subseteq\CF$, we define its width at a context $X$ as
\begin{equation}
\label{eq:confWidths}
w_{\widetilde{\CF}}(X)
=
\sup_{\underline{f},\overline{f}\in\widetilde{\CF}}
\left(
\overline{f}(X)-\underline{f}(X)
\right).
\end{equation}
This quantity measures the uncertainty of the function class at the current context and plays the same role as confidence widths in the linear analysis.

The $T$-round constrained pseudo-regret is
\begin{equation}
\label{eq:nonlinear_regret}
\mathcal{R}(T)
=
\sum_{t=1}^{T}
\left[
g(\alpha_t\opt)\theta\opt(X_t)
-
g(\alpha_t)\theta\opt(X_t)
\right],
\end{equation}
where
\[
\alpha_t\opt
\in
\arg\max_{\alpha\in\mathcal{A}_t^f(X_t)}
g(\alpha)\theta\opt(X_t)
\]
is the optimal feasible action at round $t$.

We define the dataset available at the beginning of round $t$ as
\[
\mathcal{D}_t
=
\left\{
\left(X_s,\frac{R_s}{g(\alpha_s)},\frac{C_s}{g(\alpha_s)}\right)
:
s<t,\ g(\alpha_s)>0
\right\}.
\]
Thus, whenever $g(\alpha_s)>0$,
\[
\frac{R_s}{g(\alpha_s)}
=
\theta\opt(X_s)+\xi_s^r,
\qquad
\frac{C_s}{g(\alpha_s)}
=
\mu\opt(X_s)+\xi_s^c.
\]
Rounds with $g(\alpha_s)=0$ provide no information about the unknown functions and are excluded.
For any function $f$, let
\[
\norm{f}_{\mathcal{D}_t}
=
\sqrt{
\sum_{(X_s,\cdot,\cdot)\in\mathcal{D}_t}
f^2(X_s)
}
\]
denote the empirical norm induced by the observed contexts.

At each round, we construct least-squares estimates $\hat{\theta}_t$ and $\hat{\mu}_t$ over the classes $\CF^r$ and $\CF^c$ using the transformed observations in $\mathcal{D}_t$.
We then define the confidence sets
\begin{align}
\label{eq:nonlinear_conf_sets}
\mathcal{C}_t^r(\delta')
&=
\left\{
\theta\in\CF^r:
\norm{\theta-\hat{\theta}_t}_{\mathcal{D}_t}^2
\le
\rho^r(t,\delta')
\right\},
\\
\mathcal{C}_t^c(\delta')
&=
\left\{
\mu\in\CF^c:
\norm{\mu-\hat{\mu}_t}_{\mathcal{D}_t}^2
\le
\rho^c(t,\delta')
\right\}.
\end{align}
For finite function classes, one may take
\[
\rho^r(t,\delta')
=
8\useconstant{Cr}^2
\log\left(\frac{|\CF^r|}{\delta'}\right),
\qquad
\rho^c(t,\delta')
=
8\useconstant{Cc}^2
\log\left(\frac{|\CF^c|}{\delta'}\right).
\]
More generally, the regret bound below is stated in terms of the corresponding confidence radii and eluder dimensions.

To construct the estimated feasible action set, we choose the most conservative cost function in the confidence set at the current context.
Specifically, we solve
\begin{equation}
\label{eq:mu_convex_prog_eluder}
\begin{aligned}
    & \underset{\mu\in\CF^c}{\max}
    &&  \mu(X_t) \\
    & \text{subject to}
    &&
    \mu\in\mathcal{C}_t^c(\delta'),
    \\
    &&
    &
    \mu(X_t)
    +
    \useconstant{Cc}\sqrt{2\log\left(\frac{1}{\delta}\right)}
    \ge
    0 .
\end{aligned}
\end{equation}
If this optimization problem has no feasible solution, we set
\[
\hat{\mathcal{A}}_t^f=[0,1].
\]
Otherwise, let $\tilde{\mu}_t$ denote a solution of \eqref{eq:mu_convex_prog_eluder}.
The estimated feasible set is then
\begin{equation}
\label{eq:nonlinear_estimated_feasible_set}
\hat{\mathcal{A}}_t^f
=
\left[
0,
g^{-1}
\left(
\min\left\{
1,
\frac{\tau}{
\tilde{\mu}_t(X_t)
+
\useconstant{Cc}\sqrt{2\log\left(\frac{1}{\delta}\right)}
}
\right\}
\right)
\right].
\end{equation}

For the reward function, we use optimism and select
\[
\tilde{\theta}_t(X_t)
=
\max_{\theta\in\mathcal{C}_t^r(\delta')}
\theta(X_t).
\]
The action is then chosen as
\[
\alpha_t
\in
\arg\max_{\alpha\in\hat{\mathcal{A}}_t^f}
g(\alpha)\tilde{\theta}_t(X_t).
\]
Since $g$ is strictly increasing, this optimization reduces to selecting the largest action in $\hat{\mathcal{A}}_t^f$ whenever $\tilde{\theta}_t(X_t)\ge0$, and selecting $\alpha_t=0$ otherwise.

\begin{algorithm}[H]
\caption{Non-Linear High Probability Constrained UCB}
\label{alg:non-linear}
\begin{algorithmic}[1]
\Require Constraint threshold $\tau>0$, safety parameter $\delta$, confidence parameter $\delta'$, sub-Gaussianity constants $\useconstant{Cr},\useconstant{Cc}$
\State $\alpha_1 \leftarrow
g^{-1}\!\left(
\min\left\{
1,
\frac{\tau}{
\useconstant{Cc}\sqrt{2\log(1/\delta)}+1
}
\right\}
\right)$
\State Select action $\alpha_1$ and store $(X_1,R_1,C_1)$
\For{$t=2,3,\ldots,T$}
    \State Compute $\hat{\theta}_t,\hat{\mu}_t$ using least-squares estimators over $\CF^r$ and $\CF^c$.
    \State Construct $\mathcal{C}_t^r(\delta')$ and $\mathcal{C}_t^c(\delta')$ using~\eqref{eq:nonlinear_conf_sets}.
    \State Compute $\hat{\mathcal{A}}_t^f$ using~\eqref{eq:mu_convex_prog_eluder} and~\eqref{eq:nonlinear_estimated_feasible_set}.
    \State Compute $\tilde{\theta}_t(X_t)=\max_{\theta\in\mathcal{C}_t^r(\delta')}\theta(X_t)$.
    \State Compute
    \[
    \alpha_t
    \in
    \arg\max_{\alpha\in\hat{\mathcal{A}}_t^f}
    g(\alpha)\tilde{\theta}_t(X_t).
    \]
    \State Select action $\alpha_t$.
    \If{$g(\alpha_t)>0$}
        \State Store $(X_t,R_t,C_t)$.
    \EndIf
\EndFor
\end{algorithmic}
\end{algorithm}

Let $d^r_{\mathrm{eluder}}$ and $d^c_{\mathrm{eluder}}$ denote the eluder dimensions of the reward and cost function classes, respectively.
The following theorem gives the regret guarantee for the non-linear setting.

\begin{restatable}{thm}{eluderRegBound}
\label{thm:eluderRegBound}
With probability at least $1-\delta'$, the regret of the \textit{Non-Linear High Probability Constrained UCB} algorithm satisfies
\begin{align*}
\mathcal{R}(T)
=
\mathcal{O}\Big(
&
\sqrt{
d^r_{\mathrm{eluder}}T\rho^r(T,\delta'/3)
}
\\
&+
\frac{1}{\tau}
\sqrt{
d^c_{\mathrm{eluder}}T\rho^c(T,\delta'/3)
}
\Big).
\end{align*}
\end{restatable}

\section{Experimental Results}
\label{sec:experiments}

We evaluate the performance of \Cref{alg:HPUCB} on both synthetic and real-world datasets.
A natural application of our algorithm is dosage selection, where the action corresponds to the administered treatment level.
However, access to large-scale medical datasets with rich patient context is limited, and we leave a thorough empirical study in clinical settings for future work.
Instead, we assess the performance of our method on the \href{https://data.nasa.gov/dataset/li-ion-battery-aging-datasets}{NASA Li-ion Battery Aging Datasets}.
In this setting, we interpret the current level as the dosage, the voltage change as the reward, and the temperature as the cost.

We first report in \Cref{tab:hpucb_onexpectation_tau} the violation ratio of \texttt{HPUCB} and compare it with a baseline algorithm that enforces the cost constraint only in expectation, as in \citep{pacchiano2024contextual}.
We observe that \texttt{HPUCB} achieves a violation ratio of $0.02\%$, whereas the expected-cost baseline exhibits violation rates between $12\%$ and $23\%$.
Such violation rates are undesirable in safety-critical applications, where individual unsafe outcomes may be unacceptable.

We begin with synthetic experiments.
We generate $\theta^*$ and $\mu^*$ by sampling from a $d$-dimensional standard normal distribution and normalizing them to unit norm.
Contexts $X_t$ are generated similarly from a multivariate normal distribution and normalized.
We consider dimensions $d \in \{5,10\}$ and multiple threshold values $\tau$, and run each configuration for $10^4$ rounds over 5 random seeds.
In the linear case, we use the closed-form solutions derived for \Cref{eq:theta_hat,eq:mu_convex_prog}, avoiding the need to solve convex programs at each round.
The total running time across all seeds is below two minutes for all thresholds.
All experiments were run on a MacBook Air with an Apple M3 chip and 24~GB RAM using the JAX framework.
The dominant computational cost comes from Cholesky decompositions, which scale as $\CO(d^3)$.
JAX substantially accelerates these linear algebra operations, especially when executed on a GPU.
This becomes particularly relevant in the real-world experiment, where the feature dimension is 2048.

\begin{table}[ht]
\centering
\caption{Comparison of HP-UCB and On-Expectation across different values of $\tau$, with 95\% confidence intervals.}
\label{tab:hpucb_onexpectation_tau}
\begin{tabular}{ccc}
\hline
$\tau$ & HP-UCB (95\% CI) & On-Expectation (95\% CI) \\
\hline
0.1 & 0.0002, [0.0001, 0.0002] & 0.2392, [0.2354, 0.2430] \\
0.2 & 0.0002, [0.0001, 0.0002] & 0.2313, [0.2272, 0.2355] \\
0.3 & 0.0002, [0.0001, 0.0002] & 0.2224, [0.2179, 0.2269] \\
0.4 & 0.0002, [0.0001, 0.0002] & 0.2126, [0.2077, 0.2176] \\
0.5 & 0.0002, [0.0001, 0.0002] & 0.2016, [0.1964, 0.2067] \\
0.6 & 0.0002, [0.0001, 0.0002] & 0.1890, [0.1834, 0.1946] \\
0.7 & 0.0002, [0.0001, 0.0002] & 0.1763, [0.1705, 0.1821] \\
0.8 & 0.0002, [0.0001, 0.0002] & 0.1621, [0.1563, 0.1678] \\
0.9 & 0.0002, [0.0001, 0.0002] & 0.1472, [0.1419, 0.1525] \\
1.0 & 0.0002, [0.0001, 0.0002] & 0.1324, [0.1274, 0.1374] \\
\hline
\end{tabular}
\end{table}

In \Cref{fig:synthetic}, we report representative learning curves for $d \in \{5,10\}$ with threshold $\tau=0.5$.
We compare our algorithm with a variant of $\varepsilon$-Greedy that plays the safe action $\alpha_1$, as defined in \Cref{eq:initialSafeAction}, with probability $\varepsilon$, and otherwise follows the policy of \Cref{alg:HPUCB}.
The main advantage of this variant is its lower computational cost, especially in higher dimensions, while still collecting approximately $\varepsilon T$ informative samples that help tighten its confidence intervals.
Empirically, we find that $\varepsilon=0.5$ provides a favorable trade-off between computational efficiency and statistical performance, achieving comparable regret while avoiding dependence on $T$ in the parameter tuning.

A direct comparison with the method of \cite{pacchiano2024contextual} is not entirely appropriate, as their algorithm enforces the cost constraint only in expectation, whereas our setting requires high-probability constraint satisfaction for the realized cost.
As shown in \Cref{tab:hpucb_onexpectation_tau}, this difference leads to substantially higher violation rates for the expected-cost baseline.

\begin{figure}[ht]
    \centering
    \begin{subfigure}{0.48\columnwidth}
        \centering
        \includegraphics[width=\linewidth]{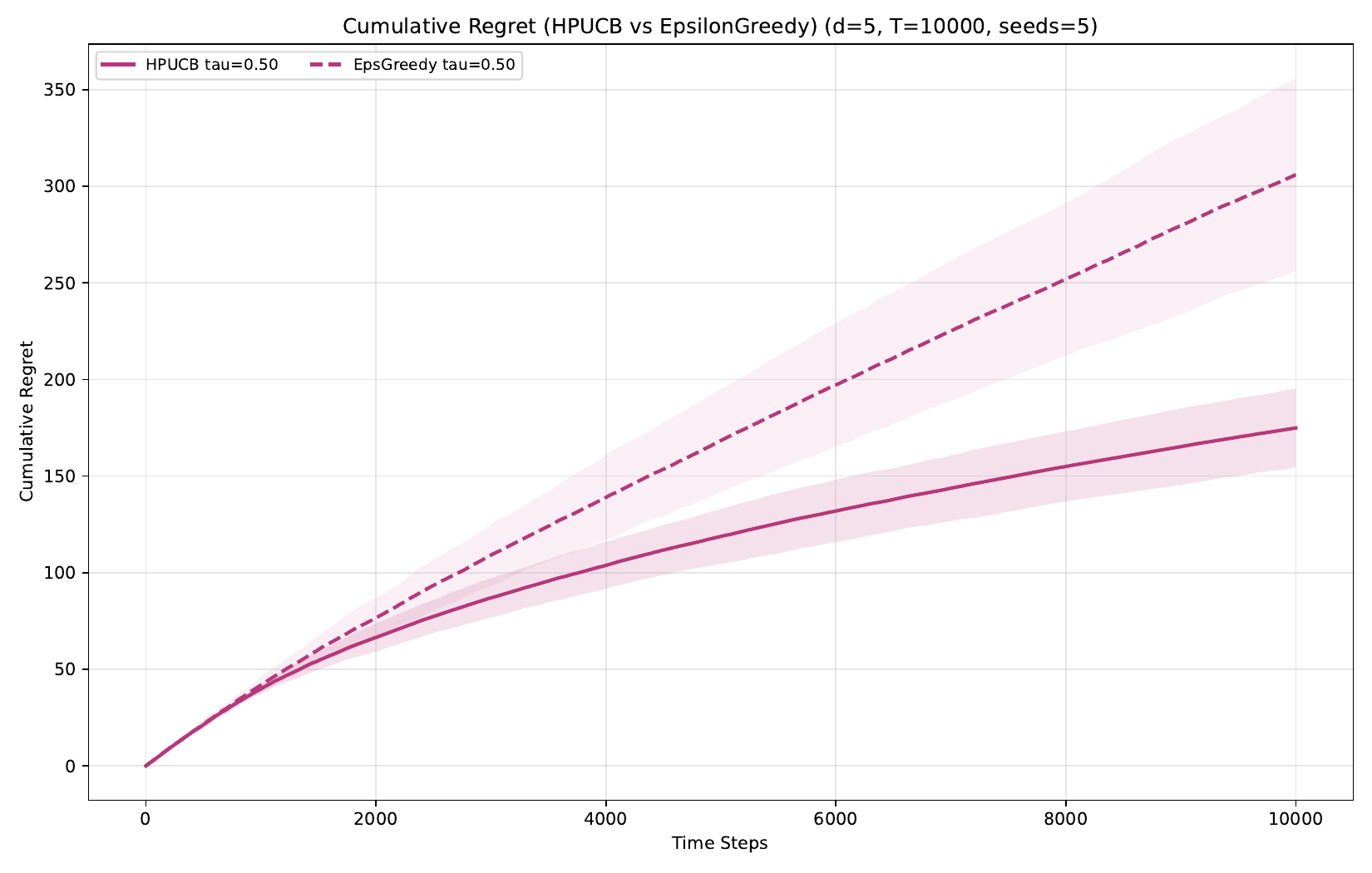}
        \caption{$d=5$, $\tau=0.5$}
        \label{fig:syntheticd5}
    \end{subfigure}
    \hfill
    \begin{subfigure}{0.48\columnwidth}
        \centering
        \includegraphics[width=\linewidth]{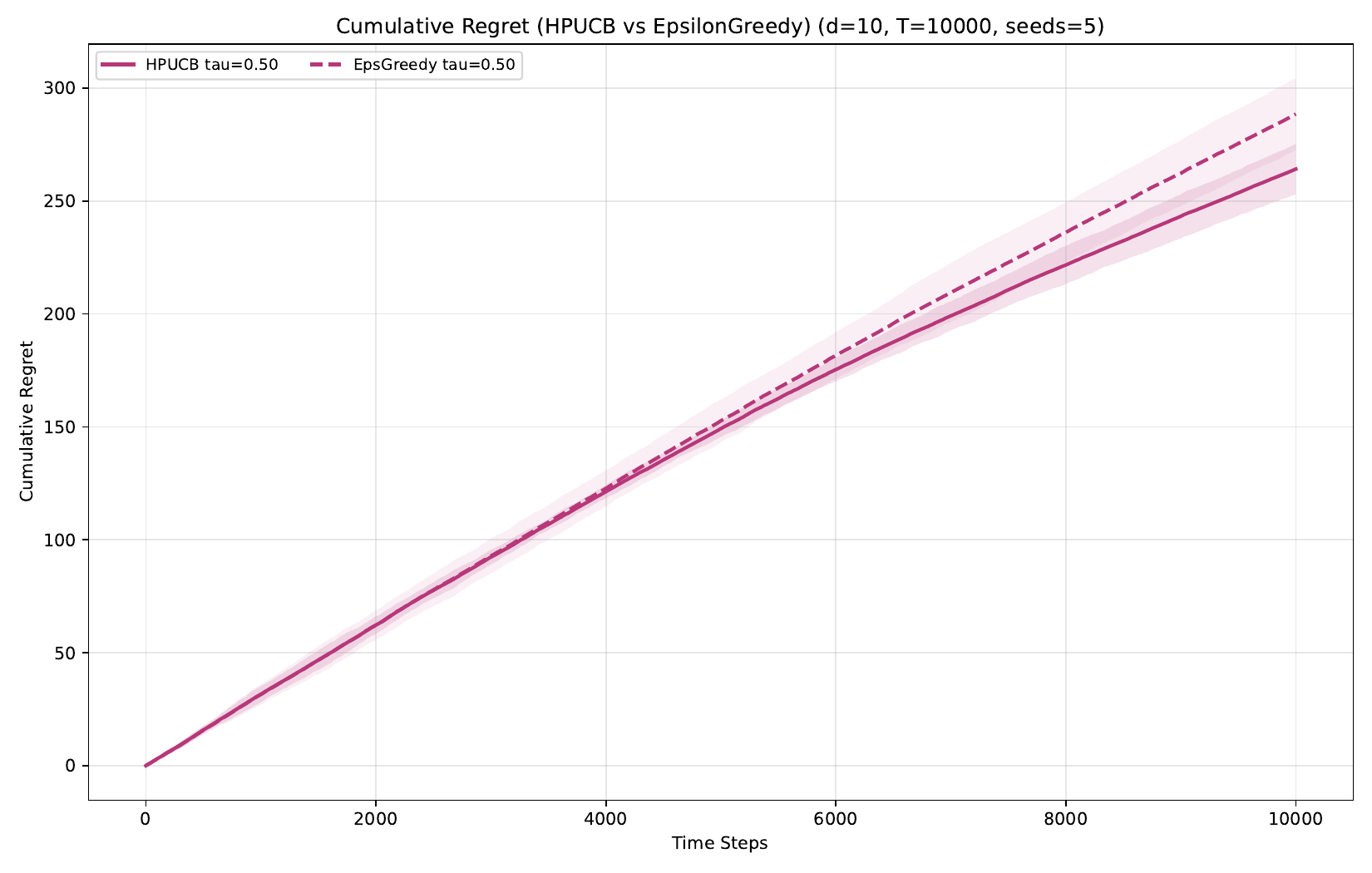}
        \caption{$d=10$, $\tau=0.5$}
        \label{fig:syntheticd10}
    \end{subfigure}
    \caption{Synthetic experiments comparing \texttt{HPUCB} and $\varepsilon$-Greedy for $\tau=0.5$.}
    \label{fig:synthetic}
\end{figure}

We now describe the experimental setup for the \href{https://data.nasa.gov/dataset/li-ion-battery-aging-datasets}{NASA Li-ion Battery Aging Datasets}.
The dataset contains measurements from four types of batteries (RW9, RW10, RW11, RW12) operating under both charging and discharging modes, across a range of current levels.
For each cycle, the dataset records the voltage change, temperature, textual metadata describing the operating conditions, and a time series corresponding to the duration of the cycle.

To construct contextual features, we use embeddings generated by the language model Gemma \citep{team2024gemma}.
At each round, the input to the language model consists of all available information from the previous five time steps, formatted as a JSON object.
The resulting embedding vector has dimension 2048 and serves as the context $X_t$.

To satisfy the realizability assumption and obtain ground-truth parameters $\theta^*$ and $\mu^*$, we split the dataset into four parts and use approximately half of the data, about 100k samples with dimension $2048+3$, to perform a least-squares fit.
The fitted parameters are then treated as the true underlying model, and we evaluate the performance of \Cref{alg:HPUCB} on the remaining portion of the dataset.

Due to the high dimensionality of the problem, all real-world experiments were conducted on a compute cluster equipped with an NVIDIA L40 GPU.
The total running time was approximately 30 minutes for \texttt{HPUCB} and about two minutes for the baseline algorithm, which has access to the true reward and cost parameters.
We report the cumulative regret of \texttt{HPUCB} and $\varepsilon$-Greedy for threshold values $\tau \in \{0.1,0.2\}$.

\begin{figure}[ht]
    \centering
    \begin{subfigure}{0.48\columnwidth}
        \centering
        \includegraphics[width=\linewidth]{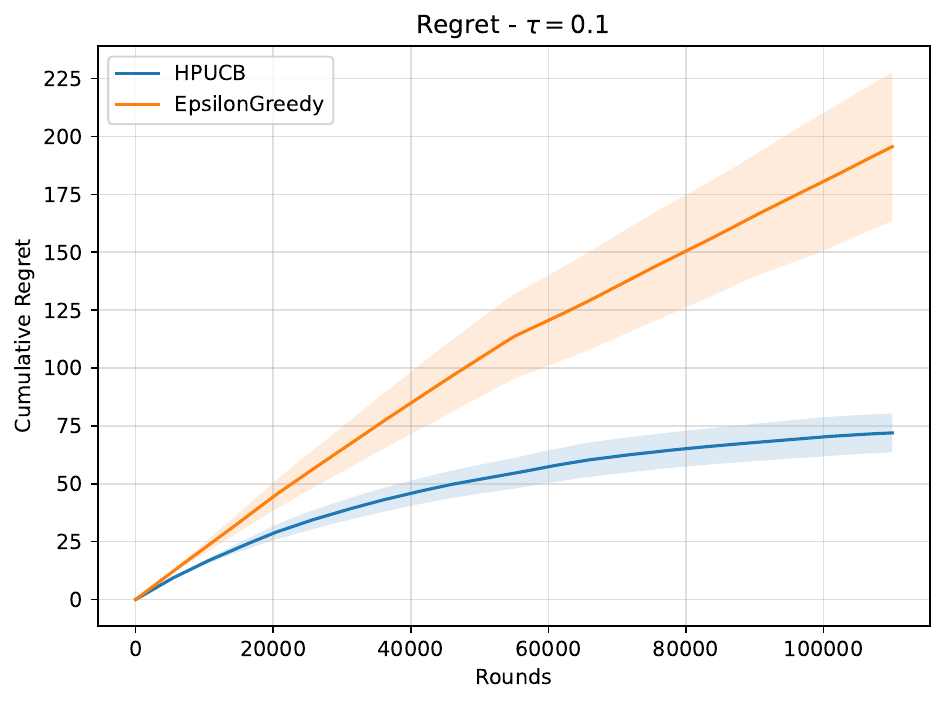}
        \caption{$\tau=0.1$}
        \label{fig:nasaPoint1}
    \end{subfigure}
    \hfill
    \begin{subfigure}{0.48\columnwidth}
        \centering
        \includegraphics[width=\linewidth]{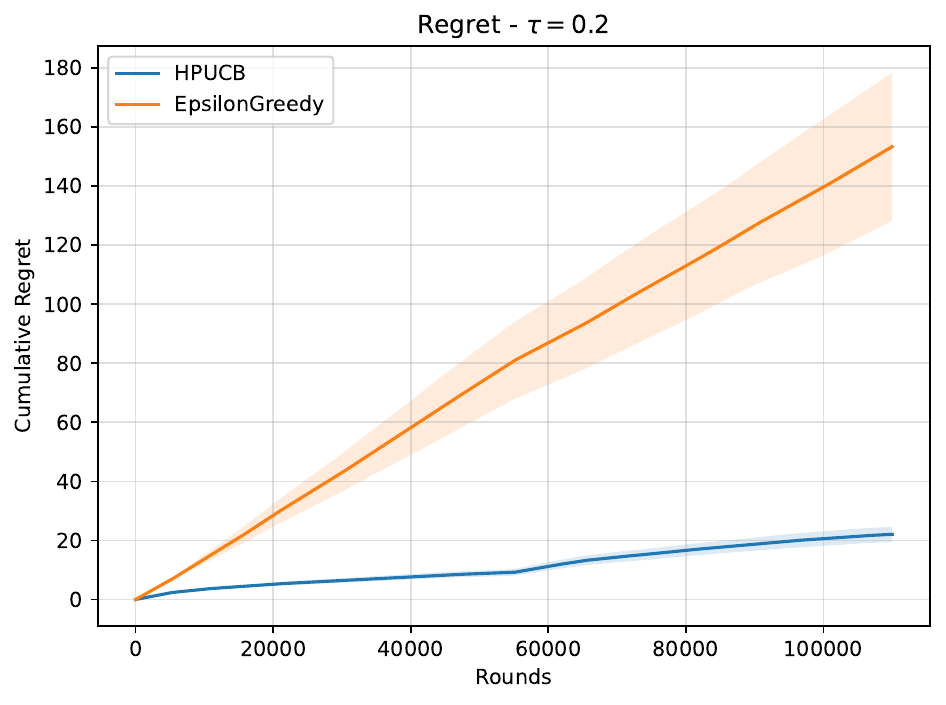}
        \caption{$\tau=0.2$}
        \label{fig:nasaPoint2}
    \end{subfigure}
    \caption{NASA Li-ion Battery Aging Dataset experiments for different safety thresholds.}
    \label{fig:nasa}
\end{figure}

Finally, we note that the guarantees of our algorithm hold for arbitrary context sequences $X_t$, including both stochastic and adversarial settings.
This property makes the proposed approach particularly appealing for applications such as battery charging and control problems, where an explicit dynamical model of the system may be unavailable or difficult to specify.
In such scenarios, our algorithm provides a principled way to perform safe decision-making directly from observed data.

\section{Conclusions}
\label{sec:conclusions}

We studied contextual bandits with continuous actions under stage-wise high-probability constraints on the realized cost, motivated by applications in which the variability of the outcome depends on the selected action.
We introduced an optimistic--pessimistic algorithm that balances reward exploration with conservative estimation of the feasible action set, and established regret guarantees for both linear and general reward--cost models.
Our experimental results further illustrate the importance of accounting for heteroscedasticity when safety is imposed on realized, rather than expected, costs.

Several directions remain open.
An important extension is to consider more general forms of heteroscedasticity beyond the action-scaled structure studied here, as well as Bayesian approaches to the same realized-cost constrained problem.
Another promising direction is a more extensive empirical evaluation in safety-critical real-world applications, particularly adaptive clinical dosage selection.

\bibliography{ref}
\bibliographystyle{plainnat}

\newpage
\appendix

  \hsize\textwidth
  \linewidth\hsize  {\centering
  {\Large\bfseries Appendix \par}}
  \vskip 0.2in
\tableofcontents
 \addcontentsline{toc}{section}{Supplementary Material}
\newpage
\appendix

\section{Constraint Formulation}
\label{sec:constraintForm}

\subsection{Proof of \Cref{lem:constraintFormulation}}
\label{app:constraint}

By \Cref{ass:noise-sub-gaussian}, the cost noise $\xi_t^c$ is conditionally $\useconstant{Cc}$-sub-Gaussian given the history $\mathcal{H}_t$.
Let $z_t := g(\alpha_t)$.
If $z_t=0$, then $C_t=0$, and since $\tau>0$ the constraint is satisfied trivially.
We therefore consider the case $z_t>0$.

\begin{thm}[Sub-Gaussian concentration bounds -- Theorem 5.3 in \cite{lattimore2020bandit}]
\label{thm:subgaussian_bound}
If $X$ is $\sigma$-sub-Gaussian, then for any $\epsilon>0$,
\[
\P(X\ge \epsilon)\le \exp\left(-\frac{\epsilon^2}{2\sigma^2}\right).
\]
\end{thm}

\constraintFormulation*

\begin{proof}
Since
\[
C_t
=
z_t\left(\langle X_t,\mu\opt\rangle+\xi_t^c\right),
\]
we have
\[
\frac{C_t-z_t\langle X_t,\mu\opt\rangle}{z_t}
=
\xi_t^c.
\]
Thus, conditional on $\mathcal{H}_t$,
\begin{align*}
\P(C_t\ge \tau\mid\mathcal{H}_t)
&=
\P\left(
\xi_t^c
\ge
\frac{\tau-z_t\langle X_t,\mu\opt\rangle}{z_t}
\,\middle|\,
\mathcal{H}_t
\right) \\
&\le
\exp\left(
-
\frac{
\left(
\dfrac{\tau-z_t\langle X_t,\mu\opt\rangle}{z_t}
\right)^2
}{
2\useconstant{Cc}^2
}
\right),
\end{align*}
where the last step follows from \Cref{thm:subgaussian_bound}.
Therefore, it is sufficient to require
\[
\useconstant{Cc}\sqrt{2\log\left(\frac{1}{\delta}\right)}
\le
\frac{\tau-z_t\langle X_t,\mu\opt\rangle}{z_t},
\]
which is equivalent to
\[
g(\alpha_t)
\left(
\langle X_t,\mu\opt\rangle
+
\useconstant{Cc}\sqrt{2\log\left(\frac{1}{\delta}\right)}
\right)
\le
\tau.
\]
Under this condition, $\P(C_t\le \tau\mid\mathcal{H}_t)\ge 1-\delta$.
\end{proof}

\subsection{Proof of \Cref{lem:feasibleSet}}
\label{app:feasibleSet}

\feasibleSet*

\begin{proof}
Let
\[
B_t(\mu)
:=
\langle X_t,\mu\rangle
+
\useconstant{Cc}\sqrt{2\log\left(\frac{1}{\delta}\right)}.
\]
On the event $\{\mu\opt\in\CC_t^c(\delta')\}$, which holds with probability at least $1-\delta'$ by \Cref{thm:yasin_theorem}, we show that
\[
\hat{\mathcal{A}}_t^f(\delta)
\subseteq
\mathcal{A}_t^f(\delta).
\]

We consider two cases.

\textbf{Case 1: $\KappaM_t=\emptyset$.}
Since $\mu\opt\in\CC_t^c(\delta')$, the only way the feasible set of \Cref{eq:mu_convex_prog} can be empty is if
\[
B_t(\mu\opt)<0.
\]
Then for every $\alpha\in[0,1]$,
\[
g(\alpha)B_t(\mu\opt)\le 0\le \tau.
\]
Hence $\mathcal{A}_t^f(\delta)=[0,1]$.
By construction, $\hat{\mathcal{A}}_t^f(\delta)=[0,1]$ as well, so the inclusion holds.

\textbf{Case 2: $\KappaM_t\neq\emptyset$.}
Let $\tilde{\mu}_t$ be the solution of \Cref{eq:mu_convex_prog}.
If $B_t(\mu\opt)<0$, then again $\mathcal{A}_t^f(\delta)=[0,1]$, and the inclusion is immediate.

It remains to consider the case $B_t(\mu\opt)\ge 0$.
Since $\mu\opt\in\KappaM_t$ and $\tilde{\mu}_t$ maximizes $\langle X_t,\mu\rangle$ over $\KappaM_t$, we have
\[
\langle X_t,\mu\opt\rangle
\le
\langle X_t,\tilde{\mu}_t\rangle,
\]
and therefore
\[
B_t(\mu\opt)\le B_t(\tilde{\mu}_t).
\]
Now take any $\alpha\in\hat{\mathcal{A}}_t^f(\delta)$.
By definition,
\[
g(\alpha)B_t(\tilde{\mu}_t)\le \tau.
\]
Since $g(\alpha)\ge 0$ and $B_t(\mu\opt)\le B_t(\tilde{\mu}_t)$, it follows that
\[
g(\alpha)B_t(\mu\opt)\le \tau.
\]
Thus $\alpha\in\mathcal{A}_t^f(\delta)$.
Therefore,
\[
\hat{\mathcal{A}}_t^f(\delta)
\subseteq
\mathcal{A}_t^f(\delta).
\]
Conditioning on $\{\mu\opt\in\CC_t^c(\delta')\}$ gives the desired probability bound.
\end{proof}

\section{Good Event}
\label{app:good_event}

\subsection{Proof of \Cref{lem:goodEvent}}

\goodEvent*

\begin{proof}
We prove that each component of the good event holds with probability at least $1-\delta'$, and then apply a union bound.
By \Cref{thm:yasin_theorem},
\[
\P(\Varepsilon_\theta)\ge 1-\delta',
\qquad
\P(\Varepsilon_\mu)\ge 1-\delta'.
\]
Moreover, \Cref{lem:feasibleSet} gives
\[
\P(\Varepsilon_A)\ge 1-\delta'.
\]
Therefore,
\[
\P(\Varepsilon)
=
\P(\Varepsilon_\theta\cap\Varepsilon_\mu\cap\Varepsilon_A)
\ge
1-3\delta',
\]
as claimed.
\end{proof}

\section{Analytical Solutions to \Cref{eq:theta_hat,eq:mu_convex_prog}}
\label{sec:KKT_derivation}

In this section, we derive closed-form solutions to \Cref{eq:theta_hat,eq:mu_convex_prog} in the linear case using the Karush--Kuhn--Tucker (KKT) conditions \citep{boyd2004convex}.
This avoids solving a convex program at each round and substantially reduces the computational cost of the algorithm.

\begin{lemma}[Linear maximization over an ellipsoid]
\label{lem:linear_max_ellipsoid}
Let $c\in\mathbb{R}^d$ with $c\neq 0$, let $\Sigma\in\mathbb{R}^{d\times d}$ be symmetric positive definite, and let $\beta>0$.
Consider
\[
\max_{x\in\mathbb{R}^d}\ \langle c,x\rangle
\quad\text{subject to}\quad
x^\top \Sigma x \le \beta^2 .
\]
Then the unique optimal solution is
\[
x^\star
=
\beta\,
\frac{\Sigma^{-1}c}{\norm{c}_{\Sigma^{-1}}},
\]
and the optimal value is
\[
\max_{x^\top \Sigma x\le \beta^2}
\langle c,x\rangle
=
\beta\norm{c}_{\Sigma^{-1}}.
\]
\end{lemma}

\begin{proof}
Since $\Sigma\succ 0$, the feasible set is nonempty, compact, and convex.
The objective is continuous and linear, so an optimum exists.
The KKT conditions are necessary and sufficient because Slater's condition holds.
The Lagrangian is
\[
\mathcal{L}(x,\lambda)
=
\langle c,x\rangle
-
\lambda(x^\top\Sigma x-\beta^2),
\qquad
\lambda\ge0.
\]
Stationarity gives
\[
c-2\lambda\Sigma x=0,
\qquad\text{hence}\qquad
x=\frac{1}{2\lambda}\Sigma^{-1}c.
\]
Since $c\neq0$, the optimum lies on the boundary, so $x^\top\Sigma x=\beta^2$.
Substituting the expression for $x$ gives
\[
\beta^2
=
\frac{1}{4\lambda^2}c^\top\Sigma^{-1}c.
\]
Thus
\[
\lambda
=
\frac{1}{2\beta}
\sqrt{c^\top\Sigma^{-1}c}.
\]
Substituting back yields
\[
x^\star
=
\beta\frac{\Sigma^{-1}c}{\sqrt{c^\top\Sigma^{-1}c}}
=
\beta\frac{\Sigma^{-1}c}{\norm{c}_{\Sigma^{-1}}}.
\]
The optimal value follows immediately.
Uniqueness follows from strict convexity of the ellipsoid boundary in the direction of the nonzero linear functional.
\end{proof}

\begin{proposition}[Closed form for $\tilde{\theta}_t$]
\label{prop:KKT_theta}
The solution of \Cref{eq:theta_hat} is uniquely given by
\[
\tilde{\theta}_t
=
\hat{\theta}_t
+
\beta_t^r(\delta',d)
\frac{\Sigma_t^{-1}X_t}{\norm{X_t}_{\Sigma_t^{-1}}}.
\]
Moreover,
\[
\max_{\theta\in\CC_t^r(\delta')}
\langle X_t,\theta\rangle
=
\langle X_t,\hat{\theta}_t\rangle
+
\beta_t^r(\delta',d)
\norm{X_t}_{\Sigma_t^{-1}}.
\]
\end{proposition}

\begin{proof}
Apply \Cref{lem:linear_max_ellipsoid} with
$c=X_t$ and $x=\theta-\hat{\theta}_t$.
\end{proof}

\begin{proposition}[Closed form for $\tilde{\mu}_t$]
\label{prop:KKT_mu}
The solution of
\[
\max_{\mu\in\CC_t^c(\delta')}
\langle X_t,\mu\rangle
\]
is uniquely given by
\[
\tilde{\mu}_t
=
\hat{\mu}_t
+
\beta_t^c(\delta',d)
\frac{\Sigma_t^{-1}X_t}{\norm{X_t}_{\Sigma_t^{-1}}}.
\]
Moreover,
\[
\max_{\mu\in\CC_t^c(\delta')}
\langle X_t,\mu\rangle
=
\langle X_t,\hat{\mu}_t\rangle
+
\beta_t^c(\delta',d)
\norm{X_t}_{\Sigma_t^{-1}}.
\]
\end{proposition}

\begin{proof}
Apply \Cref{lem:linear_max_ellipsoid} with
$c=X_t$ and $x=\mu-\hat{\mu}_t$.
\end{proof}

\begin{lemma}[Feasibility of the pessimistic cost program]
\label{lem:cost_program_feasibility}
Let
\[
\CC_t^c
=
\left\{
\mu\in\mathbb{R}^d:
\norm{\mu-\hat{\mu}_t}_{\Sigma_t}\le \beta_t^c
\right\},
\]
and define
\[
b_\delta
=
\useconstant{Cc}\sqrt{2\log\left(\frac{1}{\delta}\right)}.
\]
Consider
\[
\max_{\mu\in\CC_t^c}
\langle X_t,\mu\rangle
\qquad
\text{subject to}
\qquad
\langle X_t,\mu\rangle+b_\delta\ge 0.
\]
Let
\[
\bar{v}_t
=
\max_{\mu\in\CC_t^c}
\langle X_t,\mu\rangle
=
\langle X_t,\hat{\mu}_t\rangle
+
\beta_t^c\norm{X_t}_{\Sigma_t^{-1}}.
\]
Then the program is infeasible if and only if $\bar{v}_t+b_\delta<0$.
If $\bar{v}_t+b_\delta\ge0$, then the unconstrained maximizer over $\CC_t^c$ is feasible and is also the optimizer of the constrained problem.
\end{lemma}

\begin{proof}
If $\bar{v}_t+b_\delta<0$, then for every $\mu\in\CC_t^c$,
\[
\langle X_t,\mu\rangle+b_\delta
\le
\bar{v}_t+b_\delta
<0.
\]
Thus the constraint cannot be satisfied, and the program is infeasible.

If $\bar{v}_t+b_\delta\ge0$, then the unconstrained maximizer satisfies the additional constraint.
Since it already maximizes $\langle X_t,\mu\rangle$ over the larger set $\CC_t^c$, it also maximizes the same objective over the constrained feasible subset.
\end{proof}

\section{Regret Analysis -- Linear Case}
\label{sec:regretAnalysis}

As shown in \Cref{sec:reg_analysisLinear}, we decompose the regret into two terms: the cost of estimating $\theta\opt$ and the cost of estimating $\mu\opt$.
We now bound these terms in terms of the problem parameters.

\subsection{Proof of \Cref{lem:reg_first_term_bound}}
\label{appen:reg_first_term_bound_proof}

\regFirstTermBound*

\begin{proof}
Condition on the event $\Varepsilon_\theta$, so that $\theta\opt\in\CC_t^r$ for all $t\in[T]$.
Since $g(\alpha_t\opt)\ge0$, we have
\begin{align*}
g(\alpha_t\opt)\langle X_t,\theta\opt\rangle
&\le
\max_{\theta\in\CC_t^r}
g(\alpha_t\opt)\langle X_t,\theta\rangle \\
&\le
\max_{\alpha\in\mathcal{A}_t^f}
\max_{\theta\in\CC_t^r}
g(\alpha)\langle X_t,\theta\rangle \\
&=
\max_{\alpha\in\mathcal{A}_t^f}
g(\alpha)\langle X_t,\tilde{\theta}_t\rangle \\
&=
g(\tilde{\alpha}_t)\langle X_t,\tilde{\theta}_t\rangle .
\end{align*}
Therefore,
\[
\left(g(\alpha_t\opt)-g(\tilde{\alpha}_t)\right)
\langle X_t,\theta\opt\rangle
\le
g(\tilde{\alpha}_t)
\langle X_t,\tilde{\theta}_t-\theta\opt\rangle .
\]
\end{proof}

\subsection{Proof of \Cref{lem:reg_sec_term_bound}}
\label{appen:reg_sec_term_bound}

\regSecTermBound*

\begin{proof}
Let
\[
b_\delta
=
\useconstant{Cc}\sqrt{2\log\left(\frac{1}{\delta}\right)},
\qquad
B_t^\star
=
\langle X_t,\mu\opt\rangle+b_\delta,
\qquad
\widetilde{B}_t
=
\langle X_t,\tilde{\mu}_t\rangle+b_\delta.
\]
On the event $\Varepsilon_\mu$, the pessimistic construction gives
\[
B_t^\star\le \widetilde{B}_t
\]
whenever $B_t^\star\ge0$.
If $B_t^\star<0$, then the true feasible set is all of $[0,1]$.

Since $g$ is strictly increasing, maximizing
\[
g(\alpha)\langle X_t,\tilde{\theta}_t\rangle
\]
over an interval is equivalent to choosing the largest feasible action when
$\langle X_t,\tilde{\theta}_t\rangle\ge0$, and choosing $\alpha=0$ otherwise.
If $\langle X_t,\tilde{\theta}_t\rangle<0$, then both $\tilde{\alpha}_t$ and $\alpha_t$ are zero, and the claim is immediate.

Assume now that $\langle X_t,\tilde{\theta}_t\rangle\ge0$.
Define the effective actions
\[
\tilde{z}_t:=g(\tilde{\alpha}_t),
\qquad
z_t:=g(\alpha_t).
\]
The true and estimated feasible sets in the effective action variable are intervals with upper endpoints
\[
\bar{z}_t^\star
=
\begin{cases}
1, & B_t^\star\le0,\\[0.25em]
\min\left\{1,\dfrac{\tau}{B_t^\star}\right\}, & B_t^\star>0,
\end{cases}
\]
and
\[
\widehat{z}_t
=
\begin{cases}
1, & \widetilde{B}_t\le0,\\[0.25em]
\min\left\{1,\dfrac{\tau}{\widetilde{B}_t}\right\}, & \widetilde{B}_t>0.
\end{cases}
\]
Thus,
\[
\tilde{z}_t=\bar{z}_t^\star,
\qquad
z_t=\widehat{z}_t.
\]

We claim that
\[
\bar{z}_t^\star-\widehat{z}_t
\le
\frac{\widetilde{B}_t-B_t^\star}{\tau}
=
\frac{\langle X_t,\tilde{\mu}_t-\mu\opt\rangle}{\tau}.
\]

If $B_t^\star\le0$, then $\bar{z}_t^\star=1$.
If $\widehat{z}_t=1$, the claim is trivial.
Otherwise, $\widehat{z}_t=\tau/\widetilde{B}_t$, which implies $\widetilde{B}_t\ge\tau$.
Then
\[
1-\frac{\tau}{\widetilde{B}_t}
=
\frac{\widetilde{B}_t-\tau}{\widetilde{B}_t}
\le
\frac{\widetilde{B}_t-B_t^\star}{\tau},
\]
where we used $B_t^\star\le0\le\tau$ and $\widetilde{B}_t\ge\tau$.

If $B_t^\star>0$, then by pessimism $B_t^\star\le \widetilde{B}_t$.
If either upper endpoint equals $1$, the desired inequality is immediate.
Otherwise,
\[
\bar{z}_t^\star-\widehat{z}_t
=
\frac{\tau}{B_t^\star}
-
\frac{\tau}{\widetilde{B}_t}
=
\frac{\tau(\widetilde{B}_t-B_t^\star)}
{B_t^\star\widetilde{B}_t}.
\]
In this subcase both upper endpoints are strictly below $1$, hence
$B_t^\star\ge\tau$ and $\widetilde{B}_t\ge\tau$.
Therefore,
\[
\bar{z}_t^\star-\widehat{z}_t
\le
\frac{\widetilde{B}_t-B_t^\star}{\tau}.
\]

Combining the cases gives
\[
g(\tilde{\alpha}_t)-g(\alpha_t)
\le
\frac{\langle X_t,\tilde{\mu}_t-\mu\opt\rangle}{\tau}.
\]
By Cauchy--Schwarz and \Cref{ass:bounded-reward-cost-param,ass:bounded-contexts},
\[
\langle X_t,\theta\opt\rangle
\le
\norm{X_t}_2\norm{\theta\opt}_2
\le
\useconstant{S}\useconstant{L}.
\]
Thus,
\[
\left(g(\tilde{\alpha}_t)-g(\alpha_t)\right)
\langle X_t,\theta\opt\rangle
\le
\useconstant{S}\useconstant{L}
\cdot
\frac{\langle X_t,\tilde{\mu}_t-\mu\opt\rangle}{\tau}.
\]
\end{proof}

\subsection{Proof of \Cref{theorem:regret_bound}}
\label{appen:proof_of_regret_bound}

As usual in linear bandit analysis, we use the elliptic potential lemma.
The only difference is that our algorithm collects informative samples only in rounds for which $g(\alpha_t)>0$.
Since the number of such rounds is at most $T$, the standard bound applies.

\begin{lemma}[Elliptic potential lemma]
\label{lem:elliptic_potential}
Let $X_1,\ldots,X_{N(T)}$ be the contexts collected by \Cref{alg:HPUCB} in rounds with $g(\alpha_t)>0$.
For any regularization parameter $0<\lambda<1$,
\[
\sum_{t=1}^{T}
\norm{X_t}_{\Sigma_t^{-1}}^2
\le
d\log\left(1+\frac{T\useconstant{L}^2}{d}\right).
\]
\end{lemma}

By Cauchy--Schwarz,
\[
\sum_{t=1}^{T}
\norm{X_t}_{\Sigma_t^{-1}}
\le
\sqrt{
dT\log\left(1+\frac{T\useconstant{L}^2}{d}\right)
}.
\]

\regretBoundLinear*

\begin{proof}
On the good event,
\begin{align*}
\mathcal{R}_{\mathcal{C}}(T)
&=
\sum_{t=1}^{T}
\left(g(\alpha_t\opt)-g(\tilde{\alpha}_t)\right)
\langle X_t,\theta\opt\rangle
+
\sum_{t=1}^{T}
\left(g(\tilde{\alpha}_t)-g(\alpha_t)\right)
\langle X_t,\theta\opt\rangle
\\
&\le
\sum_{t=1}^{T}
g(\tilde{\alpha}_t)
\langle X_t,\tilde{\theta}_t-\theta\opt\rangle
+
\frac{\useconstant{S}\useconstant{L}}{\tau}
\sum_{t=1}^{T}
\langle X_t,\tilde{\mu}_t-\mu\opt\rangle
\\
&\le
\sum_{t=1}^{T}
g(\tilde{\alpha}_t)
\norm{X_t}_{\Sigma_t^{-1}}
\norm{\tilde{\theta}_t-\theta\opt}_{\Sigma_t}
+
\frac{\useconstant{S}\useconstant{L}}{\tau}
\sum_{t=1}^{T}
\norm{X_t}_{\Sigma_t^{-1}}
\norm{\tilde{\mu}_t-\mu\opt}_{\Sigma_t}
\\
&\le
2
\sum_{t=1}^{T}
\beta_t^r(\delta'/3,d)
\norm{X_t}_{\Sigma_t^{-1}}
+
\frac{2\useconstant{S}\useconstant{L}}{\tau}
\sum_{t=1}^{T}
\beta_t^c(\delta'/3,d)
\norm{X_t}_{\Sigma_t^{-1}}
\\
&\le
2
\left(
1+\frac{\useconstant{S}\useconstant{L}}{\tau}
\right)
\max\{\beta_T^r(\delta'/3,d),\beta_T^c(\delta'/3,d)\}
\sum_{t=1}^{T}
\norm{X_t}_{\Sigma_t^{-1}}
\\
&\le
2
\left(
1+\frac{\useconstant{S}\useconstant{L}}{\tau}
\right)
\max\{\beta_T^r(\delta'/3,d),\beta_T^c(\delta'/3,d)\}
\sqrt{
dT\log\left(1+\frac{T\useconstant{L}^2}{d}\right)
}.
\end{align*}
In the third inequality, we used $g(\tilde{\alpha}_t)\le1$ and the fact that both $\tilde{\theta}_t,\theta\opt$ and $\tilde{\mu}_t,\mu\opt$ lie in their corresponding confidence sets.
\end{proof}

\section{Non-linear Rewards and Costs}
\label{app:nonlinear}

\subsection{Concentration for Non-linear Least Squares}

For the non-linear case, define the informative set
\[
\CS_t
=
\{s\le t: g(\alpha_s)>0\}.
\]
The least-squares estimators are
\begin{equation}
\label{eq:thetaLSENL}
\hat{\theta}_t
:=
\argmin_{\theta\in\CF^r}
\sum_{s\in\CS_{t-1}}
\left(
\theta(X_s)-\frac{R_s}{g(\alpha_s)}
\right)^2,
\end{equation}
and
\begin{equation}
\label{eq:muLSENL}
\hat{\mu}_t
:=
\argmin_{\mu\in\CF^c}
\sum_{s\in\CS_{t-1}}
\left(
\mu(X_s)-\frac{C_s}{g(\alpha_s)}
\right)^2.
\end{equation}
Whenever $g(\alpha_s)>0$,
\[
\frac{R_s}{g(\alpha_s)}
=
\theta\opt(X_s)+\xi_s^r,
\qquad
\frac{C_s}{g(\alpha_s)}
=
\mu\opt(X_s)+\xi_s^c.
\]
Thus, after normalization by $g(\alpha_s)$, the regression observations have the same conditionally sub-Gaussian noise as in the original model.

For any function $f$, define
\[
\norm{f}_{\mathcal{D}_t}
=
\sqrt{
\sum_{s\in\CS_{t-1}} f^2(X_s)
}.
\]
The confidence sets are
\begin{align}
\label{eq:nonlinear_conf_sets_appendix}
\mathcal{C}^r_t(\delta')
&=
\left\{
\theta\in\CF^r:
\norm{\theta-\hat{\theta}_t}_{\mathcal{D}_t}^2
\le
\rho^r(t,\delta')
\right\},
\\
\mathcal{C}^c_t(\delta')
&=
\left\{
\mu\in\CF^c:
\norm{\mu-\hat{\mu}_t}_{\mathcal{D}_t}^2
\le
\rho^c(t,\delta')
\right\}.
\end{align}
For finite function classes, we take
\[
\rho^r(t,\delta')
=
8\useconstant{Cr}^2
\log\left(\frac{|\CF^r|}{\delta'}\right),
\qquad
\rho^c(t,\delta')
=
8\useconstant{Cc}^2
\log\left(\frac{|\CF^c|}{\delta'}\right).
\]

\begin{lemma}
\label{lem:villesConcentration}
For any sequence of real-valued random variables $(Z_t)_{t\le T}$ adapted to a filtration $(\mathcal{F}_t)_{t\le T}$, it holds with probability at least $1-\delta$ that, for all $T'\le T$,
\[
\sum_{t=1}^{T'} Z_t
\le
\sum_{t=1}^{T'}
\log\left(
\E_{t-1}\left[e^{Z_t}\right]
\right)
+
\log(\delta^{-1}),
\]
where $\E_{t-1}[\cdot]=\E[\cdot\mid\mathcal{F}_{t-1}]$.
\end{lemma}

\begin{proof}
Define
\[
M_\tau
=
\exp\left(
\sum_{t=1}^{\tau}
Z_t
-
\sum_{t=1}^{\tau}
\log\left(
\E_{t-1}\left[e^{Z_t}\right]
\right)
\right).
\]
Then $(M_\tau)_{\tau\le T}$ is a nonnegative martingale with $M_0=1$.
Ville's inequality gives
\[
\P\left(\exists \tau\le T: M_\tau>\frac{1}{\delta}\right)
\le
\delta.
\]
Taking logarithms and complements yields the claim.
\end{proof}

\begin{proposition}
\label{prop:concentrationArg}
For any fixed $\eta\in\mathbb{R}$, with probability at least $1-\delta$, for all $\theta\in\CF^r$ and all $T'\le T$,
\[
\eta
\sum_{t=1}^{T'}
\xi_t^r
\left(
\theta\opt(X_t)-\theta(X_t)
\right)
\le
\frac{\eta^2\useconstant{Cr}^2}{2}
\sum_{t=1}^{T'}
\left(
\theta\opt(X_t)-\theta(X_t)
\right)^2
+
\log\left(\frac{|\CF^r|}{\delta}\right).
\]
Similarly, with probability at least $1-\delta$, for all $\mu\in\CF^c$ and all $T'\le T$,
\[
\eta
\sum_{t=1}^{T'}
\xi_t^c
\left(
\mu\opt(X_t)-\mu(X_t)
\right)
\le
\frac{\eta^2\useconstant{Cc}^2}{2}
\sum_{t=1}^{T'}
\left(
\mu\opt(X_t)-\mu(X_t)
\right)^2
+
\log\left(\frac{|\CF^c|}{\delta}\right).
\]
\end{proposition}

\begin{proof}
Fix $\theta\in\CF^r$ and apply \Cref{lem:villesConcentration} with
\[
Z_t
=
\eta\xi_t^r
\left(
\theta\opt(X_t)-\theta(X_t)
\right).
\]
By conditional sub-Gaussianity,
\[
\log
\E_{t-1}
\left[
e^{Z_t}
\right]
\le
\frac{\eta^2\useconstant{Cr}^2}{2}
\left(
\theta\opt(X_t)-\theta(X_t)
\right)^2.
\]
Thus, for this fixed $\theta$, the desired inequality holds with probability at least $1-\delta$.
A union bound over $\theta\in\CF^r$ gives the stated reward result.
The cost result is identical, replacing $\theta$ by $\mu$ and $\useconstant{Cr}$ by $\useconstant{Cc}$.
\end{proof}

\begin{lemma}[Concentration for non-linear rewards and costs]
\label{lem:confidRad}
Let $\delta'\in(0,1)$.
With probability at least $1-2\delta'$, for all $t\le T$,
\[
\theta\opt\in\mathcal{C}_t^r(\delta'),
\qquad
\mu\opt\in\mathcal{C}_t^c(\delta').
\]
\end{lemma}

\begin{proof}
We prove the reward statement; the cost statement is identical.
By the definition of $\hat{\theta}_t$ as a least-squares estimator,
\[
\sum_{s\in\CS_{t-1}}
\left(
\hat{\theta}_t(X_s)-\frac{R_s}{g(\alpha_s)}
\right)^2
\le
\sum_{s\in\CS_{t-1}}
\left(
\theta\opt(X_s)-\frac{R_s}{g(\alpha_s)}
\right)^2.
\]
Since
\[
\frac{R_s}{g(\alpha_s)}
=
\theta\opt(X_s)+\xi_s^r,
\]
this implies
\[
\sum_{s\in\CS_{t-1}}
\left(
\hat{\theta}_t(X_s)-\theta\opt(X_s)
\right)^2
\le
2
\sum_{s\in\CS_{t-1}}
\xi_s^r
\left(
\hat{\theta}_t(X_s)-\theta\opt(X_s)
\right).
\]
Applying \Cref{prop:concentrationArg} with
$\theta=\hat{\theta}_t$ and choosing
$\eta=1/(2\useconstant{Cr}^2)$ yields
\[
\sum_{s\in\CS_{t-1}}
\left(
\hat{\theta}_t(X_s)-\theta\opt(X_s)
\right)^2
\le
8\useconstant{Cr}^2
\log\left(\frac{|\CF^r|}{\delta'}\right).
\]
Thus $\theta\opt\in\mathcal{C}_t^r(\delta')$.
The same argument gives $\mu\opt\in\mathcal{C}_t^c(\delta')$.
A union bound gives probability at least $1-2\delta'$.
\end{proof}

\subsection{Non-linear Regret Analysis}

Define the nonlinear good event
\[
\Varepsilon_{\mathrm{nl}}
:=
\left\{
\theta\opt\in\mathcal{C}_t^r(\delta'),
\;
\mu\opt\in\mathcal{C}_t^c(\delta'),
\;
\hat{\mathcal{A}}_t^f\subseteq\mathcal{A}_t^f,
\;
\forall t\in[T]
\right\}.
\]
By the preceding concentration result and the same feasible-set argument as in \Cref{lem:feasibleSet},
\[
\P(\Varepsilon_{\mathrm{nl}})\ge 1-3\delta'.
\]

For a subset $\widetilde{\CF}\subseteq\CF$, recall the confidence width
\[
w_{\widetilde{\CF}}(X)
=
\sup_{\underline{f},\overline{f}\in\widetilde{\CF}}
\left(
\overline{f}(X)-\underline{f}(X)
\right).
\]

\begin{lemma}
\label{lem:nonLfirstTermReg}
Under the nonlinear good event $\Varepsilon_{\mathrm{nl}}$,
\[
\sum_{t=1}^{T}
\left[
g(\alpha_t\opt)\theta\opt(X_t)
-
g(\tilde{\alpha}_t)\theta\opt(X_t)
\right]
\le
\sum_{t=1}^{T}
w_{\mathcal{C}_t^r}(X_t).
\]
\end{lemma}

\begin{proof}
For each $t$, since $\theta\opt\in\mathcal{C}_t^r$ and $g(\alpha_t\opt)\ge0$,
\begin{align*}
g(\alpha_t\opt)\theta\opt(X_t)
&\le
\max_{\theta\in\mathcal{C}_t^r}
g(\alpha_t\opt)\theta(X_t) \\
&\le
\max_{\alpha\in\mathcal{A}_t^f}
\max_{\theta\in\mathcal{C}_t^r}
g(\alpha)\theta(X_t) \\
&=
g(\tilde{\alpha}_t)\tilde{\theta}_t(X_t).
\end{align*}
Therefore,
\[
g(\alpha_t\opt)\theta\opt(X_t)
-
g(\tilde{\alpha}_t)\theta\opt(X_t)
\le
g(\tilde{\alpha}_t)
\left(
\tilde{\theta}_t(X_t)-\theta\opt(X_t)
\right).
\]
Since $g(\tilde{\alpha}_t)\le1$ and both $\tilde{\theta}_t,\theta\opt\in\mathcal{C}_t^r$,
\[
g(\tilde{\alpha}_t)
\left(
\tilde{\theta}_t(X_t)-\theta\opt(X_t)
\right)
\le
w_{\mathcal{C}_t^r}(X_t).
\]
Summing over $t$ proves the claim.
\end{proof}

\begin{lemma}
\label{lem:nonLSecTermReg}
Under the nonlinear good event $\Varepsilon_{\mathrm{nl}}$,
\[
\sum_{t=1}^{T}
\left[
g(\tilde{\alpha}_t)\theta\opt(X_t)
-
g(\alpha_t)\theta\opt(X_t)
\right]
\le
\frac{1}{\tau}
\sum_{t=1}^{T}
w_{\mathcal{C}_t^c}(X_t).
\]
\end{lemma}

\begin{proof}
Fix a round $t$ and define
\[
b_\delta
=
\useconstant{Cc}\sqrt{2\log\left(\frac{1}{\delta}\right)},
\qquad
B_t^\star
=
\mu\opt(X_t)+b_\delta,
\qquad
\widetilde{B}_t
=
\tilde{\mu}_t(X_t)+b_\delta.
\]
By construction of the pessimistic cost estimate and on the good event,
\[
B_t^\star\le \widetilde{B}_t
\]
whenever $B_t^\star\ge0$.

If the optimistic reward estimate at $X_t$ is negative, then both $\tilde{\alpha}_t$ and $\alpha_t$ are zero, and the contribution to regret is nonpositive.
We therefore consider the case where the optimistic reward estimate is nonnegative.
Since $g$ is strictly increasing, the selected actions correspond to the largest feasible effective actions.

Let
\[
\tilde{z}_t=g(\tilde{\alpha}_t),
\qquad
z_t=g(\alpha_t).
\]
The true and estimated safe effective-action upper endpoints are
\[
\bar{z}_t^\star
=
\begin{cases}
1, & B_t^\star\le0,\\[0.25em]
\min\left\{1,\dfrac{\tau}{B_t^\star}\right\}, & B_t^\star>0,
\end{cases}
\]
and
\[
\widehat{z}_t
=
\begin{cases}
1, & \widetilde{B}_t\le0,\\[0.25em]
\min\left\{1,\dfrac{\tau}{\widetilde{B}_t}\right\}, & \widetilde{B}_t>0.
\end{cases}
\]
Thus,
\[
\tilde{z}_t=\bar{z}_t^\star,
\qquad
z_t=\widehat{z}_t.
\]
The same endpoint calculation as in \Cref{lem:reg_sec_term_bound} gives
\[
g(\tilde{\alpha}_t)-g(\alpha_t)
=
\tilde{z}_t-z_t
\le
\frac{\widetilde{B}_t-B_t^\star}{\tau}
=
\frac{\tilde{\mu}_t(X_t)-\mu\opt(X_t)}{\tau}.
\]
Since $\theta\opt(X_t)\le1$ by boundedness,
\[
g(\tilde{\alpha}_t)\theta\opt(X_t)
-
g(\alpha_t)\theta\opt(X_t)
\le
\frac{\tilde{\mu}_t(X_t)-\mu\opt(X_t)}{\tau}.
\]
Finally, both $\tilde{\mu}_t$ and $\mu\opt$ belong to $\mathcal{C}_t^c$ on the good event, so
\[
\tilde{\mu}_t(X_t)-\mu\opt(X_t)
\le
w_{\mathcal{C}_t^c}(X_t).
\]
Summing over $t$ gives the result.
\end{proof}

\begin{lemma}[Eluder-dimension width bound]
\label{lem:eluderLemma}
Let $\mathcal{F}$ be a function class with eluder dimension
$d_{\mathrm{eluder}}=d_{\mathrm{eluder}}(\mathcal{F},1/T^2)$, and let $\mathcal{C}_t$ be confidence sets with radius $\rho(T,\delta')$.
Then
\[
\sum_{t=1}^{T}
w_{\mathcal{C}_t}(X_t)
=
\mathcal{O}\left(
\sqrt{
d_{\mathrm{eluder}}T\rho(T,\delta')
}
\right).
\]
\end{lemma}

\subsection{Proof of \Cref{thm:eluderRegBound}}

\eluderRegBound*

\begin{proof}
Using \Cref{lem:nonLfirstTermReg,lem:nonLSecTermReg}, we have
\begin{align*}
\mathcal{R}_{\mathcal{C}}(T)
&=
\sum_{t=1}^{T}
\left[
g(\alpha_t\opt)\theta\opt(X_t)
-
g(\tilde{\alpha}_t)\theta\opt(X_t)
\right]
+
\sum_{t=1}^{T}
\left[
g(\tilde{\alpha}_t)\theta\opt(X_t)
-
g(\alpha_t)\theta\opt(X_t)
\right]
\\
&\le
\sum_{t=1}^{T}
w_{\mathcal{C}_t^r}(X_t)
+
\frac{1}{\tau}
\sum_{t=1}^{T}
w_{\mathcal{C}_t^c}(X_t).
\end{align*}
Applying \Cref{lem:eluderLemma} to the reward and cost confidence classes gives
\[
\sum_{t=1}^{T}
w_{\mathcal{C}_t^r}(X_t)
=
\mathcal{O}
\left(
\sqrt{
d^r_{\mathrm{eluder}}T\rho^r(T,\delta'/3)
}
\right),
\]
and
\[
\sum_{t=1}^{T}
w_{\mathcal{C}_t^c}(X_t)
=
\mathcal{O}
\left(
\sqrt{
d^c_{\mathrm{eluder}}T\rho^c(T,\delta'/3)
}
\right).
\]
Therefore,
\[
\mathcal{R}_{\mathcal{C}}(T)
=
\mathcal{O}
\left(
\sqrt{
d^r_{\mathrm{eluder}}T\rho^r(T,\delta'/3)
}
+
\frac{1}{\tau}
\sqrt{
d^c_{\mathrm{eluder}}T\rho^c(T,\delta'/3)
}
\right).
\]
\end{proof}

\end{document}